\pdfoutput=1

\documentclass[11pt]{article}

\usepackage{acl}

\usepackage[ruled,linesnumbered]{algorithm2e}
\usepackage{times}
\usepackage[]{mdframed}
\usepackage[]{tcolorbox}
\usepackage{latexsym}
\usepackage{amsmath}
\usepackage{physics}
\usepackage{bbding}
\usepackage{pifont}
\usepackage{wasysym}
\usepackage{amssymb}
\usepackage{bbm}
\usepackage{multirow}
\usepackage{graphicx}
\usepackage{makecell}
\usepackage{multicol,booktabs,tabularx}

\usepackage{amsmath, amssymb, amsthm}
\usepackage{mathtools}
\usepackage{booktabs}
\usepackage{hyperref}
\usepackage{enumitem}

\usepackage{xcolor}  
\usepackage{soul}  
\usepackage{array}  
\usepackage{adjustbox} 
 
\definecolor{myhlcolor}{HTML}{A0D6E5}  
\sethlcolor{myhlcolor}  

\newcolumntype{L}{>{\hspace{2pt}}l<{\hspace{2pt}}}  
\newcolumntype{C}{>{\hspace{2pt}}c<{\hspace{2pt}}}

\usepackage[T1]{fontenc}

\usepackage[utf8]{inputenc}

\usepackage{microtype}

\usepackage{inconsolata}

\usepackage{graphicx}

\newtheorem{theorem}{Theorem}

\theoremstyle{definition}
\newtheorem{assumption}{Assumption}

\theoremstyle{remark}
\newtheorem{remark}{Remark}

\newcommand{\sfreq}{\operatorname{sfreq}}
\newcommand{\E}{\mathbb{E}}

\title{\textit{Adam's Law:} Textual Frequency Law on Large Language Models}

\author{
    Hongyuan Adam Lu$^{\clubsuit}$\Thanks{\hspace{1mm}Equal Contribution.}, Z.L.$^{\clubsuit*}$, Victor Wei$^{\clubsuit}$, Zefan Zhang$^\clubsuit$, Zhao Hong$^\clubsuit$, Qiqi Xiang$^\clubsuit$
    \\\textbf{Bowen Cao$^\heartsuit$, Wai Lam$^\heartsuit$}\\
    $\clubsuit$FaceMind Corporation\\
    $\heartsuit$The Chinese University of Hong Kong\\
    hongyuanlu@outlook.com
}

\begin{document}
\maketitle
\begin{abstract}
While textual frequency has been validated as relevant to human cognition in reading speed, its relatedness to Large Language Models (LLMs) is seldom studied.  We propose a novel research direction in terms of textual data frequency, which is an understudied topic, to the best of our knowledge. Our framework is composed of three units. First, this paper proposes \textbf{T}extual \textbf{F}requency \textbf{L}aw (TFL), which indicates that frequent textual data should be preferred for LLMs for both prompting and fine-tuning. Since many LLMs are closed-source in their training data, we propose using online resources to estimate the sentence-level frequency. We then utilize an input paraphraser to paraphrase the input into a more frequent textual expression. Next, we propose \textbf{T}extual \textbf{F}requency \textbf{D}istillation (TFD) by querying LLMs to conduct story completion by further extending the sentences in the datasets, and the resulting corpora are used to adjust the initial estimation. Finally, we propose \textbf{C}urriculum \textbf{T}extual \textbf{F}requency \textbf{T}raining (CTFT) that fine-tunes LLMs in an increasing order of sentence-level frequency. Experiments are conducted on our curated dataset \textbf{T}extual \textbf{F}requency \textbf{P}aired \textbf{D}ataset (TFPD) on math reasoning, machine translation, commonsense reasoning and agentic tool calling. Results show the effectiveness of our framework.\footnote{\url{https://github.com/HongyuanLuke/frequencylaw}}
\end{abstract}
\par
\begin{figure}[ht!]
\begin{center}
\vspace{0mm}
\centerline{
\includegraphics[width=6cm]{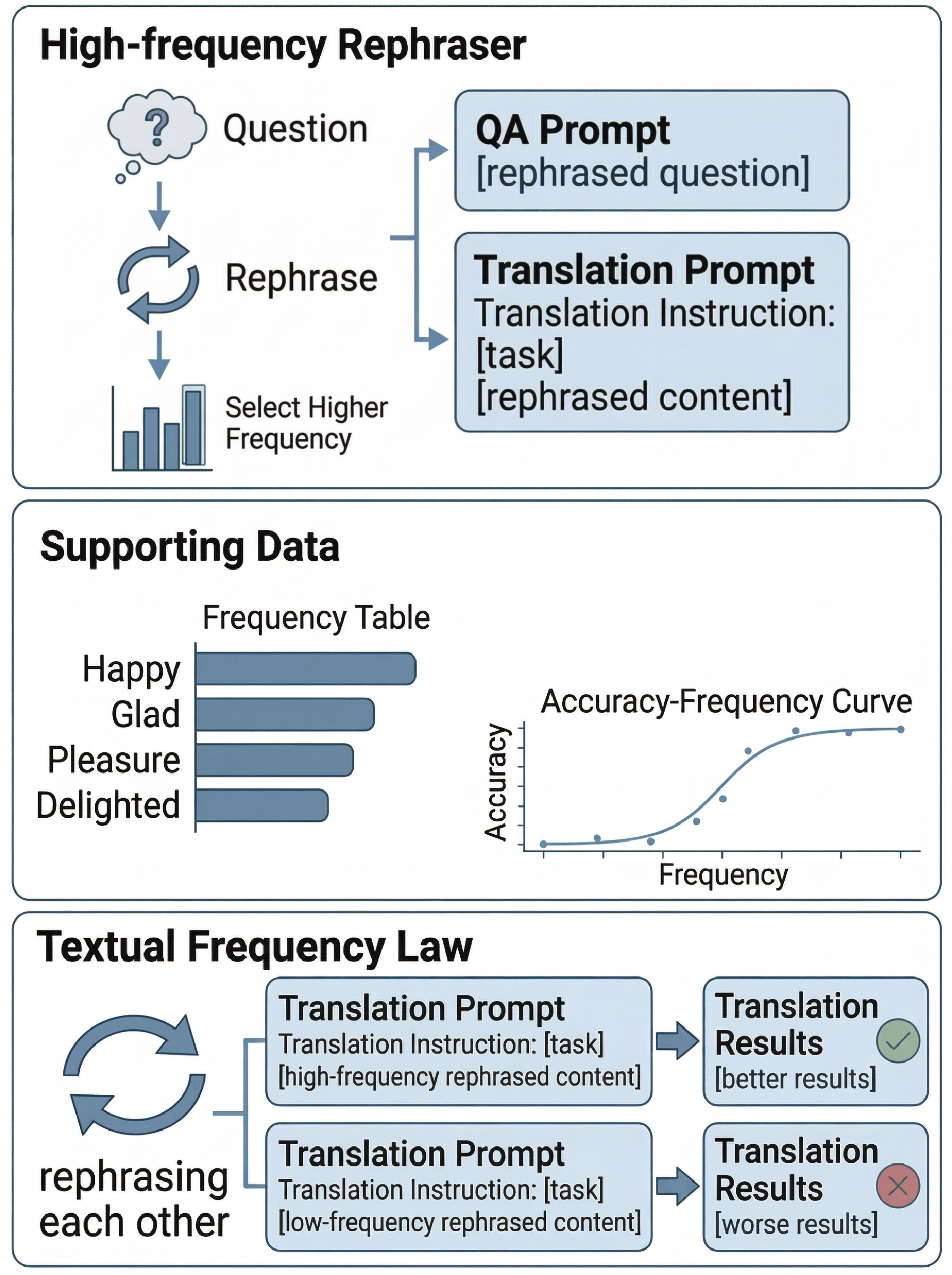}}
    \caption{\textbf{Top:} A simplified example of use case of Textual Frequency Law, where the prompt contents are rephrased and the prompt contents with higher frequency are selected. \textbf{Middle:} We achieve this by estimating sentence-level frequency with word-level frequency. \textbf{Bottom:} A toy example showing the effectiveness of our framework. Real case studies are available in the Appendix in Figure \ref{fig:casetable}. The paraphrasing can lead to semantic drift, which is the reason why human annotation is necessary in this process.}
    \label{fig:tfl}
\end{center}
\vspace{-5mm}
\end{figure}
\section{Introduction}
Large language models (LLMs) have demonstrated many exciting abilities and applications, such as chain-of-thought reasoning \citep{wang-etal-2023-towards,10.5555/3600270.3602070}, machine translation \citep{2023arXiv230506575L,zhu-etal-2024-multilingual}, and spatial reasoning \citep{hu2024chainofsymbol}, etc. More recently, increasing the length of the reasoning processes has become another popular research direction \citep{2025arXiv250112948D,2025arXiv250119393M}. Another important factor for training is the order of training, where it could be preferable from easy to hard in terms of the data difficulty \citep{lu-lam-2023-pcc}, or from short to long in terms of data length \citep{2025arXiv250215592Z}. Yet, what kind of data should be favourable during the training is an overlooked topic. Previous works have explored and concluded that the quality of the data is usually important \citep{iskander-etal-2024-quality,jin-wang-2024-select}. The amount of data is also important \citep{2024arXiv240721783G}.
\par
\citet{oh-etal-2024-frequency} found that larger models predict rare words better. In the era of LLMs, scaling factors usually mean that larger models can be better. This then may mean that predicting rare (less frequent) words could be a harder task than predicting frequent words. \citet{cao2024on} demonstrated that when prompting LLMs, different prompts with the same meaning could give very different results in terms of quality.
\par
This motivates us to investigate when the data are paraphrased to each other with the same meaning but different language expressions. The use of paraphrases has been explored in NLP research for many cases, such as mitigating data contamination \citep{zhu-etal-2024-clean}, evaluating generation tasks \citep{tang-etal-2024-metrics} and data augmentation (DA, \citet{abaskohi-etal-2023-lm}). As a DA method, paraphrases are useful for training LLMs \citep{lu-lam-2023-pcc}, so this means that we might want to include all the paraphrases in the training when it is affordable. However, training resources are usually limited, and we investigate whether the frequency matters when the meaning is kept, and the computational resources are limited for fine-tuning. Also, such investigation on paraphrased inputs into LLMs can be important, as \citet{cao2024on} has found that they usually give different performance, but there isn't a clear conclusion yet which factors are relevant to this phenomenon.
\par 
In contrast, this paper proposes novel \textbf{T}extual \textbf{F}requency \textbf{L}aw (TFL), which suggests that when the meanings are kept the same, data with higher sentence-level frequency should be preferred to the ones with low frequency, for both prompting and fine-tuning. The underlying motivation is that this paper postulates that higher-frequency data occurs more frequently than lower-frequency data in the pre-training stage, so they are easier to understand by LLMs. Based on such a law, this paper proposes to calculate the frequency estimation through online open-source data corpora, as many LLMs are closed-source and we usually do not have direct access to their training data. To further enhance the estimation, this paper proposes a novel method called \textbf{F}requency \textbf{T}extual \textbf{D}istillation. TFD conducts story completion with a text dataset on the target LLMs, and the completed story generation is used to enhance the original frequency estimation. Last, we propose \textbf{C}urriculum \textbf{T}extual \textbf{F}requency \textbf{T}raining (CTFT) that fine-tunes LLMs in increasing order of sentence-level frequency with the training data, which yields better results.
\par
 Our frequency training framework is composed of three units, and our contributions are three-fold:
\begin{itemize}
\setlength\itemsep{0em}
    \item We propose \textbf{T}extual \textbf{F}requency \textbf{L}aw, which suggests that high-frequency textual data should be preferred for LLMs when conducting prompting and fine-tuning, when the meaning of the data is kept the same, i.e., they are paraphrases.
    \item We propose a novel method called \textbf{T}extual \textbf{F}requency \textbf{D}istillation to further enhance the frequency estimation (collected from online resources) via conducting story completion to collect model generation from those LLMs that we do not have direct access to the training textual data.
    \item We propose a novel method called \textbf{C}urriculum \textbf{T}extual \textbf{F}requency \textbf{T}raining that fine-tunes LLMs in an increasing order of sentence-level frequency with the training data.
\end{itemize}
Figure \ref{fig:tfl} demonstrates a use case of our proposed framework, where prompts are rephrased to achieve higher accuracy.
\section{Prior Works}
\subsection{Textual Frequency}
Textual frequency is even related to human neural activation. \citet{PMID:33015218} explored the neural activation differences between low-frequency words and high-frequency words in reading tasks, finding that high-frequency words generally evoke stronger neural responses. \citet{a379cd8748534851afa9d2ecdd090eb4} explored the neural activation differences between low-frequency words and high-frequency words in reading tasks, finding that high-frequency words generally evoke stronger neural responses. \citet{Mohan03092019} also mentioned the impact of word frequency on semantic retrieval.\par
Then, textual frequency plays an important role in artificial intelligence.  \citet{heylen-etal-2008-modelling} investigated the semantic similarity between words of different frequencies and found that high-frequency target words have higher semantic similarity with their nearest neighbour words. \citet{oh-etal-2024-frequency} found that larger models predict rare words better. This then may mean that predicting rare (less frequent) words could be a harder task than predicting frequent words, as larger models can usually be stronger. More recnetly, \citet{lu-etal-2025-slow} discovered that low-frequency words should be explained to LLMs for better machine translation.
\subsection{Paraphrasing on Language Models}
Paraphrasing is an important language task that is tackled well by language models \citep{witteveen-andrews-2019-paraphrasing,goyal-durrett-2020-neural}. Yet, paraphrasing can still be a useful method to improve language models from various aspects. \citet{tang-etal-2024-metrics} uses paraphrases to generate diverse references, which helps in evaluating language models. \citet{zhu-etal-2024-clean} uses paraphrasing as a method to cleanly evaluate the possibly contaminated large language models. \citet{gao-etal-2020-paraphrase} uses paraphrases as data augmentation to improve goal-oriented dialogue systems. More recently, \citet{2023drllama} also uses generative data augmentation, which reflects the usefulness of paraphrasing in enhancing model performance. One setting in this paper compares the performance of LLMs on paraphrases with the same meaning but different frequencies. Yet, there are some overlooks in the previous setting. It is crucial as the computational budgets for training and prompting \citep{cao2024on} are usually limited. It raises questions: which paraphrases are more useful? Should we use all paraphrases?
\section{Proposed Approach}
\subsection{Task Formulation}
The large language model (LLM) can be regarded as a Seq2Seq neural network \citep{seq2seq} to follow the instructions to conduct various tasks with additional inputs by maximising the following likelihood:
\begin{equation}
\label{eq1}
    P\,(\mathbf{y}\mid \mathbf{i}, \mathbf{x})=\prod_{j=1}^{\mathbb{T}}P\,(y_j\mid y_1,..., y_{j-1}, \mathbf{i}, \mathbf{x}),
\end{equation}
where $\mathbb{T}$ represents the length of the generated output and $y_j$ represents the word at the position $j$ that has been inferenced. $\mathbf{i}$ represents the instruction to guide the LLMs to process the inputs. $\mathbf{x}$ represents the source sentences. Note that the actual format could be case by case for different tasks. For example, we conduct experiments on math reasoning and machine translation (MT). For math reasoning, there is no $\mathbf{x}$, as the instruction itself already contains the actual question. In contrast, for MT, $\mathbf{i}$ is usually the instruction to ask LLMs to translate the actual sentence $\mathbf{x}$ to the target language while maintaining the actual meaning. For convenience, we denote $\mathbf{x}$ as the concatenation of the instruction and the actual input in the rest of this section.
\subsection{Textual Frequency Law}
This paper proposes \textbf{T}extual \textbf{F}requency \textbf{L}aw (TFL) to select the paraphrases with the highest sentence-level textual frequency for both prompting and fine-tuning on LLMs:
\begin{equation}
\label{eq2}
    \mathrm{argmax}_{\mathbf{x}\in \mathcal{P}}(\mathrm{sfreq}(\mathbf{x},\mathcal{D})),
\end{equation}
where $\mathbf{x}$ corresponds to the textual input as in Equation \ref{eq1}. $\mathcal{P}$ represents a set of paraphrases that contain the same meaning.  $\mathrm{sfreq}$ represents a function that evaluates a sentence-level textual frequency.
\par 
Such a frequency function $\mathrm{freq}$ can be obtained and calculated given a corpus $\mathcal{D}$. In this paper, we suggest that such sentence-level frequency can be estimated by using a position-unaware multiplication of word-level frequency: 
\begin{equation}
\label{eq3}
     \mathrm{sfreq}(\mathbf{x},\mathcal{D})=\sqrt[\mathbb{K}]{\prod_{k=1}^{\mathbb{K}}\mathrm{wfreq}( \mathbf{x}_k,\mathcal{D})}
\end{equation}
\par 
Here, $\mathrm{wfreq}$ is the word-level frequency calculator that we use to estimate the sentence-level frequency. In this paper, we suggest that there is no need to obtain the actual training data of LLMs, and an arbitrary text corpus can be adapted to obtain the frequency. We obtain the sentence-level frequency with the inverse normalised multiplication of the word-level frequency.
\paragraph{Prompting} When prompting with $\mathbf{x}$, higher $\mathbf{x}$ should be used to generate outputs from LLMs.
\paragraph{Fine-tuning} For fine-tuning, $\mathbf{x}$ with a higher frequency should be used together with the desired ground truth output $\mathbf{y}$ to fine-tune the LLMs.
\subsection{Textual Frequency Distillation}
Note that the frequency we obtained in the previous section is an estimation from online resources but not the actual data, as many LLMs are closed-source in their training data. This paper proposes \textbf{T}extual \textbf{F}requency \textbf{D}istillation (TFD) to further enhance this estimation. TFD asks LLMs to generate data by the following instructions:
\begin{tcolorbox}
Please conduct story completion on the following data: <textual data>
\end{tcolorbox}
, where <textual data> represents the data we have in our training set. We denote this distilled dataset as $\mathcal{D'}$. We obtain a new frequency estimation:
\begin{equation}
    \mathcal{F}_2=\mathrm{sfreq}(\mathbf{x},\mathcal{D'}),
\end{equation}
and we denote the original frequency estimation as in Equation \ref{eq3} as $\mathcal{F}_1$. Note that this step in obtaining $\mathcal{F}_2$ is relatively computationally expensive, as the data are distilled from the actual LLMs. This is, therefore, optional, and our proposed method is still effective even with $\mathcal{F}_1$ only. We then calculate the final frequency $\mathcal{F}$ as:
\begin{equation}
    \mathcal{F}(x)=\alpha\mathcal{F}_1(x)+(1+\zeta\mathbbm{1}(\mathcal{F}_1(x)=0)) \beta\mathcal{F}_2(x),
\end{equation}
where $\alpha$, $\beta$, and $\zeta$ are hyper-parameters. In the formula above, $\zeta$ is a strengthening factor to increase the effect of the distilled frequency when the words yield an ignorable frequency in the original estimation from $\mathcal{F}_1$. The calculated frequency $\mathcal{F}(x)$ is then used to choose the highest frequency instead of the original estimated frequency as in Equation \ref{eq2} and Equation \ref{eq3}.
\subsection{Curriculum Textual Frequency Training}
Motivated by the fact that low-frequency expressions can be more diverse \citep{lu-lam-2023-pcc}, which should be trained first \citep{10.5555/2969033.2969059}, we propose \textbf{C}urriculum \textbf{T}extual \textbf{F}requency \textbf{T}raining (CTFT), a method that further uses the frequency information beyond paraphrase selection during prompting. For a training set $\mathcal{T}$ that is composed of $\mathbb{N}$ instances, we propose to arrange the data in the following training order for each epoch: 
\begin{equation}
   \mathrm{sort}_{x_n\in \mathcal{T}}(\mathcal{F}(x_n)),
\end{equation}
where $\mathrm{sort}$ is a sorting function that arranges the order from lower frequency sentence-level to higher sentence-level frequency for each training instance $x_n$ in $\mathcal{T}$ with a total number of $\mathbb{N}$ instances. Note that the training instances are usual machine learning datasets here and do not have to be paraphrases of each other. We experiment with CTFT on the fine-tuning scenarios on LLMs. CTFT extends TFL and TFD to a better fine-tuning scenario.
\begin{table}
\small
\centering
\setlength\tabcolsep{6pt}
\setlength\aboverulesep{0pt}\setlength\belowrulesep{0pt}
\setcellgapes{3pt}\makegapedcells
\begin{tabular}{l|cccc}
\hline
\textbf{Tasks}& \textbf{MR} & \textbf{MT} & \textbf{CR} & \textbf{TC} \\
\hline
\multicolumn{5}{c}{\textit{high-frequency}}\\
\hline
\#. Sentences & 738 & 526 & 575 & 114\\
Avg Length & 25.86 & 21.70 & 23.66 & 41.96\\
Max Length & 71 & 60 & 64 & 73\\
Min Length & 11 & 7 & 9 & 22\\
\hline
\multicolumn{5}{c}{\textit{low-frequency}}\\
\hline
\#. Sentences & 738 & 526 & 575 & 114\\
Avg Length & 25.28 & 24.78 & 22.43 & 47.82\\
Max Length & 59 & 62 & 57 & 86\\
Min Length & 10 & 9 & 8 & 25\\
\hline
\end{tabular}
\caption{\label{tfpd_stat}
Statistics of  \textbf{T}extual \textbf{F}requency \textbf{P}aired \textbf{D}ataset (TFPD). We denote Math Reasoning as \textbf{MR}, Machine Translation as \textbf{MT}, Commonsense Reasoning as \textbf{CR}, and Tool Calling as \textbf{TC}. We denote the total instances in the dataset as \#. Sentences, and we report the length in English words. The ground-truth answer from the original datasets is directly adopted without modification. Each sentence in the high-frequency partition is paired with one sentence in the low-frequency partition.
}
\end{table}
\subsection{Textual Frequency Paired Dataset}
There is almost no such dataset for our the goal. Therefore, we collect our own dataset, \textbf{T}extual \textbf{F}requency \textbf{P}aired \textbf{D}ataset (TFPD), for this paper. Based on the original datasets GSM8K \citep{2021arXiv211014168C}, FLORES-200 \citep{nllb2022}, CommonsenseQA \citep{talmor-etal-2019-commonsenseqa}, and ToolBench \citep{guo-etal-2024-stabletoolbench}, we use GPT-4o-mini to rephrase the English sentences in GSM8K and FLORES-200. The rephrased sentences are sent to three human annotators. For human annotation, we hired three experienced annotators who have degrees relevant to English Linguistics, paid with reasonable payment, to conduct a human validation on the generated sentences. We discard the instances if the three sentences do not have the same meaning by any human annotator. We use the following instructions to rephrase the datasets automatically: 
\begin{tcolorbox}
  My goal is to transform the original sentence into both more common and less common expressions.\\
Note: Do not omit any words such as verbs, adjectives, nouns, or adverbs.\\
You must generate two types of sentences:\\
(1) ten sentences using less common, more complex words.
\\(2) ten sentences using more common, simpler words. 
\\Return all 20 sentences directly, separated by |||| and do not use numbering.
\\Original sentence: {sentence}
\end{tcolorbox}

The above instructions on GPT-4o-mini then generate 20 paraphrases. We select the two sentences with the lowest and highest frequency, respectively, as in Equation \ref{eq1}. Those two sentences are sent along with the original input sentence for succeeding human annotation to check whether all three sentences have the same meaning:
\begin{itemize}
\setlength\itemsep{0em}
    \item The same meaning: I believe these three sentences have the same meaning.
    \item Maybe the same meaning: Maybe these three sentences have the same meaning, but I might be wrong because of some reasons, for example, some rephrased words might not be appropriate for the context.
    \item Not the same meaning: I am sure that these three sentences do not have the same meaning.
\end{itemize}
We only preserve those samples that all our annotators believe are authentically the same meaning. Finally, we obtain 738 pairs out of 1,319 original GSM8K test instances, and we obtain 526 pairs out of 1,012 original FLORES-200 dev-test instances. Note that for the fine-tuning experiments, we use the constructed TFPD dataset as the training data to check the impact of textual frequency on fine-tuning, and we randomly select 500 samples from the FLORES-200 dev set for evaluation. This process is approved by FaceMind ethics review aboard.
\par
Table \ref{tfpd_stat} presents the length statistics of the samples. For space reasons, we present frequency statistics in Appendix in Table \ref{freqstat}.

\section{Experimental Setup}
\subsection{Evaluation Metrics}
For the task of math reasoning, accuracy is adopted as the evaluation metric \citep{2021arXiv211014168C}. For the task of machine translation, we report the chrF \citep{popovic-2015-chrf} and the BLEU \citep{papineni-etal-2002-bleu} evaluations provided by the sacreBLEU repository.\footnote{https://github.com/mjpost/sacrebleu} We also adopt neural-based evaluation using COMET scores versioned wmt22-comet-da\footnote{\url{https://github.com/Unbabel/COMET}} \citep{rei-etal-2020-comet}. Note that there are 37 supported languages by COMET, out of 100 languages in this study. We release the full list as in Appendix. We use chrF signature of the parameters with nworde=6, ncorder=6, beta=2. We use BLEU signature of ngram=4, weights=(0.25, 0.25, 0.25, 0.25), smoothing=method1, smoothingfunction=SmoothingFunction().method1, tokenizer=nltkwordtokenize.
\subsection{Baselines}
We conduct experiments on both closed-source and open-source LLMs for better reproducibility on GPT-4o-mini and DeepSeek-V3 \citep{2024arXiv241219437D}. DeepSeek-V3 is an MoE model with 671B model parameters. Both of them are widely used LLMs with robust multilingual translation capabilities. We also use doubao-1.5-pro-32k and qwen2.5-7b-instruct as baselines for our translation experiments. For the fine-tuning experiments validating the effectiveness of high-frequency data and the usefulness of CTFT, all experiments are conducted on qwen2.5-7b-instruct, which is an open-source LLM. We use Llama-3.3-70B-Instruct in our MR experiments \citep{2024arXiv240721783G}. We use LoRA fine-tuning \citep{hu2022lora} throughout the paper. The hyperparameters for fine-tuning are presented in Appendix for better reproducibility.
\par 
We also compare our method for the reverse setting (fine-tuning from high-frequenty to low-frequency) as well as traditional curriculum learning (from easy-to-hard, \citep{lu-lam-2023-pcc}). For the easy-to-hard baseline, we use Max Dependency Tree Depth as the difficulty function.\footnote{We use nlp = spacy.load("en\_core\_web\_sm") to calculate it.}
\subsection{Off-the-shelf Frequency Estimation}
For off-the-shelf frequency estimation, we adopt off-the-shelf resources for estimation \footnote{https://github.com/rspeer/wordfreq} using Zipf frequency \citep{robyn_speer_2022_7199437}. Since this project is further built on many resources such as ParaCrawl \citep{banon-etal-2020-paracrawl}, we refer the readers to their projects for more references.

\subsection{Language Selection}
We randomly select 100 languages from the FLORES-200 datasets for our prompting experiments, and we release their language class according to \citet{joshi-etal-2020-state} in Table \ref{lc} in Appendix. More than half of the languages are relatively low-resource according to the class definition (class 0 or class 1). For the experiments on CTFT, we use Kabuverdianu (kea\_Latn), Kikuyu (kik\_Latn), Pangasinan (pag\_Latn), and Standard Latvian (lvs\_Latn).

\begin{figure}[thb!]
\begin{center}
\vspace{0mm}
\centerline{
\includegraphics[width=7.5cm]{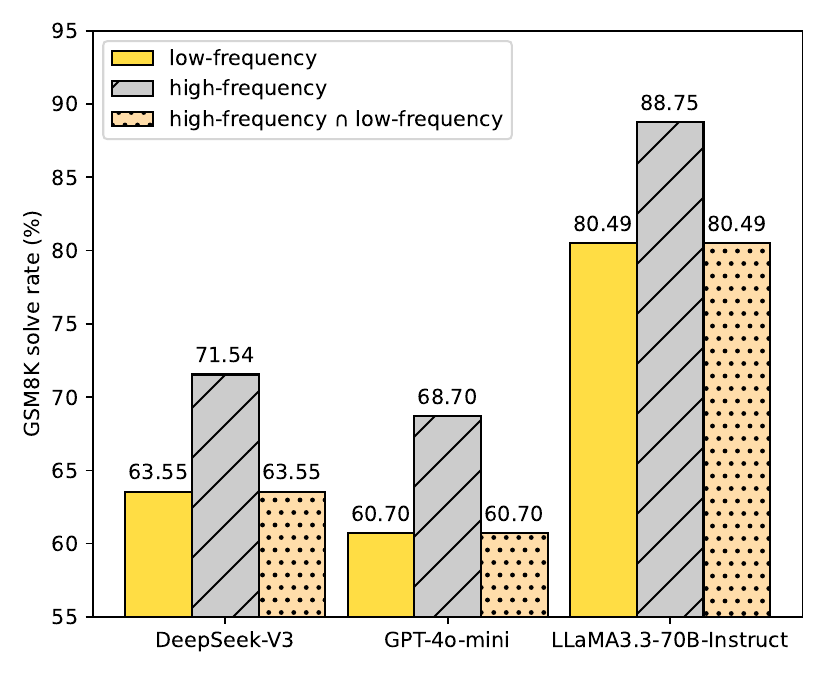}}
    \caption{The overall accuracy of TFPD on math reasoning for our proposed framework. It is obvious that the high-frequency partition in TFPD has a higher accuracy than the low-frequency partition. High-frequency $\cap$ low-frequency denotes a model that is correct in both low-frequency and high-frequency partitions.}
    \label{fig:gsm8k}
\end{center}
\vspace{-5mm}
\end{figure}

\begin{figure*}[t!]
\begin{center}
\vspace{0mm}
\centerline{
\includegraphics[width=15cm]{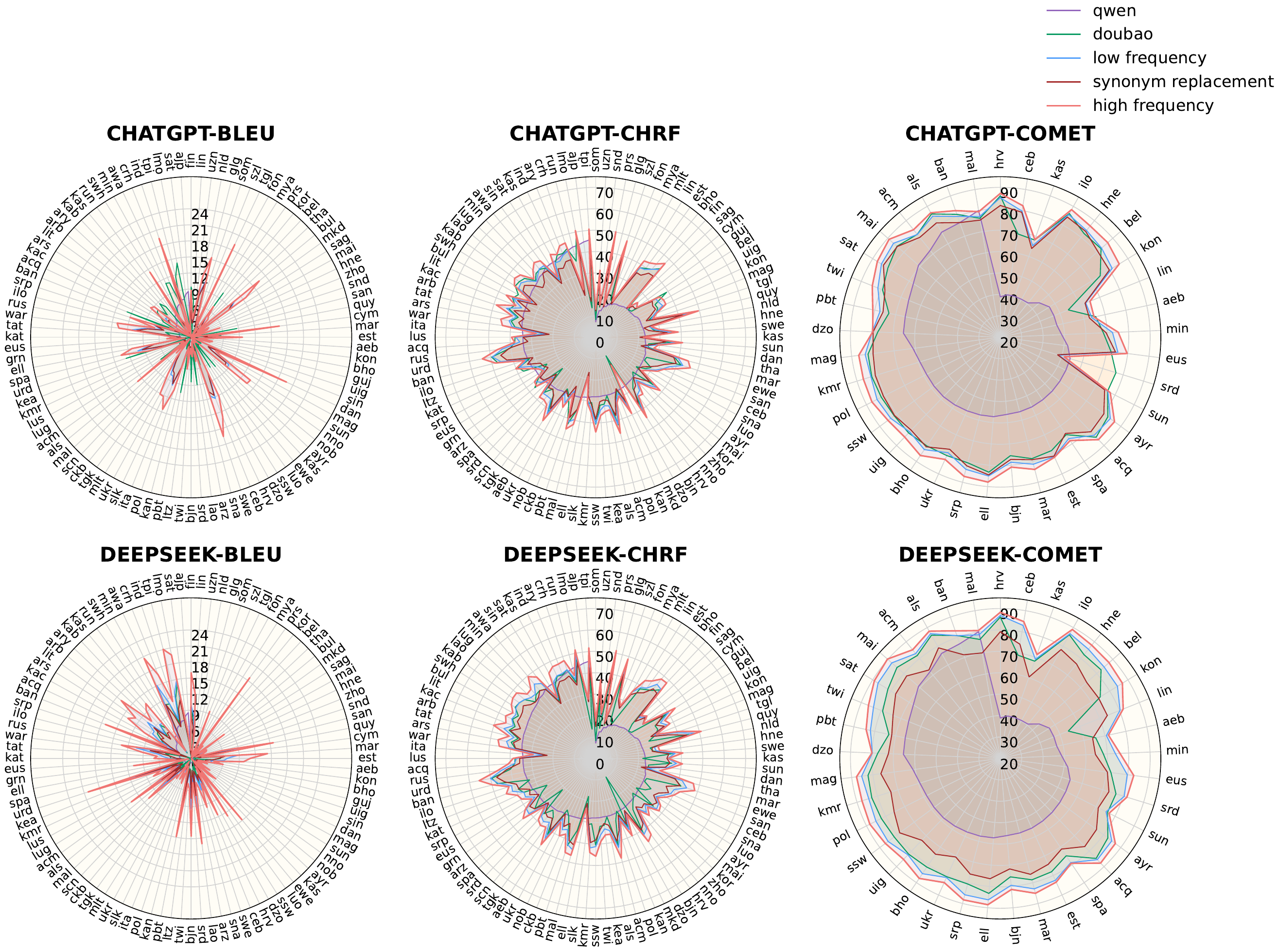}}
    \caption{The figure demonstrating the performance of our proposed framework in using high-frequency partition for translation. Results are reports on translating from English into other languages. Detailed numbers are reported in Appendix in Table \ref{tab:language_metrics1}, \ref{tab:language_metrics2}, \ref{tab:language_metrics3}, \ref{tab:language_metrics4}, \ref{tab:language_metrics5}, and \ref{tab:language_metrics6}. Synonym is a baseline that replaces words randomly with their higher-frequency rephrases using NLTK: \url{https://www.nltk.org/}.}
    \label{fig:trans}
\end{center}
\vspace{-5mm}
\end{figure*}

\subsection{Translation Prompt}
We release our 1-shot prompt for translation for better reproducibility: 

\begin{tcolorbox}
Translate the following sentence from English to \{lang\}.\\
       For example:\\
       sentence: Television reports show white smoke coming from the plant.\\  
       translation: \{trans\}\\ 
       Now, please translate the following sentence to \{lang\}.\\
       sentence: \{question\}\\
       Your output format must be like this:\\ 
       The translation result is:   
\end{tcolorbox}

\section{Results}

\begin{table*}[t!]
\centering
    \setlength\tabcolsep{10pt}
    \setlength\extrarowheight{0pt}
\begin{tabular}{l|ccc}
\hline
 
\textbf{Models} & \textbf{GPT-4o-mini} & \textbf{DeepSeek-V3} & \textbf{Llama-3.3-70B-Instruct}\\
\hline
Low-frequency partition & 0.6747 & 0.7043 &  0.7530\\
High-frequency partition & \textbf{0.6974} & \textbf{0.7235} & \textbf{0.7704}\\

\hline
\end{tabular}
\caption{\label{crresults}
Results reported in accuracy on the partition of CR. We see that the high-frequency partition gives better results on all baseline models.
}
\end{table*}
\begin{table*}[h!]
\centering
    \setlength\tabcolsep{3.5pt}
    \setlength\extrarowheight{0pt}
\begin{tabular}{l|cccc|cccc}
\hline
 
\textbf{Models} & \textbf{\# improved} & \textbf{> 1 pt} & \textbf{> 3 pts} & \textbf{> 5 pts} & \textbf{\# degraded} & \textbf{> 1 pt} & \textbf{> 3 pts} & \textbf{> 5 pts}\\
\hline
\multicolumn{9}{c}{\textit{BLEU}}\\
\hline
DeepSeek-V3   & 99/100    & 63/99     & 31/99     & 12/99      & 1/100    & 0/1      & 0/1  & 0/1 \\  
GPT-4o-mini & 95/100   & 49/95    & 27/95     & 5/95     & 5/100     & 0/5      & 0/5   & 0/5 \\   
\hline
\multicolumn{9}{c}{\textit{chrF}}\\
\hline
DeepSeek-V3   & 100/100    & 86/100     & 40/100     & 7/100      & 0/100    & 0/0      & 0/0  & 0/0 \\  
GPT-4o-mini & 91/100   & 75/91    & 34/91     & 2/91     & 9/100     & 0/9      & 0/9   & 0/9 \\ 
\hline
\multicolumn{9}{c}{\textit{COMET}}\\
\hline
DeepSeek-V3   & 37/37    & 33/37     & 4/37     & 0/37      & 0/37    & 0/0     & 0/0  & 0/0 \\  
GPT-4o-mini &  36/37    & 35/36     & 11/36     & 0/36      & 1/37    & 0/1     & 0/1  & 0/1 \\  
\hline
\end{tabular}
\caption{\label{summmt}
Statistics of the changes on prompting experiments in BLEU, chrF, and COMET scores with the high-frequency partition compared to the low-frequency partition on our established TFDP dataset. We evaluate translation from English into other languages. Most translations have been clearly improved. When there is any degradation, the degradation is less than 1 point. We denote `point' as `pt' and `points' as `pts'.
}
\end{table*}

\begin{table*}[ht!]
\centering
    \setlength\tabcolsep{10pt}
    \setlength\extrarowheight{0pt}
\begin{tabular}{l|cccc}
\hline
 
\textbf{Models} & \textbf{kea\_Latn} & \textbf{kik\_Latn} & \textbf{pag\_Latn}&\textbf{lvs\_Latn}\\
\hline
\multicolumn{5}{c}{\textit{BLEU}}\\
\hline
Original Model & 0.9346 & 1.0342 & 1.2296 & 2.2646\\
Fine-tuned Model & 4.6772 & 1.2811 & 4.5129 & 4.1954\\
Easy-to-hard Baseline & 5.1674	&1.3185	&4.4955	&3.5366\\
High-to-low Baseline &5.1179	&1.5298	&4.5365	&3.7840 \\
FT on LF w/o CTFT & 4.3899 & 1.4223 & 3.9073 & 3.2221\\
FT on 1/2 LF 1/2 HF w/o CTFT &4.7928&1.4783&4.4291&3.4787\\
FT on HF w/o CTFT & 5.2466&1.2432&3.7781&			3.9156\\
FT on HF w/ CTFT	&\textbf{5.3992}&\textbf{1.6570}&\textbf{4.9102}&\textbf{4.6027}\\ 
\hline
\multicolumn{5}{c}{\textit{chrF}}\\
\hline
Original Model&26.9844&20.6636&29.4351&33.2322\\
Fine-tuned Model&39.3714&25.6175&34.4672&34.0584\\
FT on LF w/o CTFT&39.4022&26.2465&33.9848&33.5538\\	
Easy-to-hard Baseline & 40.6414	&26.4981	&35.5396	&35.3337\\
High-to-low Baseline &41.0234	&26.5316	&35.8125	&36.1577 \\
FT on 1/2 LF 1/2 HF w/o CTFT &40.7831&26.8192&35.3375&34.2120\\
FT on HF  w/o CTFT&40.6515&26.4975&33.4990&35.0732\\
FT on HF w/ CTFT&\textbf{41.6206}&\textbf{27.7719}&\textbf{36.5285}&\textbf{37.0171}\\
\hline
\end{tabular}
\caption{\label{finetuning}
Results of fine-tuning experiments on translation from English into other languages, tested on the original FLORES-200 benchmark. Fine-tuned Model is tuned on the original FLORES-200 dataset. FT denotes fine-tuning, LF denotes low-frequency, HF denotes high-frequency, and CTFT denotes Curriculum Textual Frequency Training. 1/2 LF 1/2 HF denotes a training set with half samples sampled from the low-frequency partition and half samples sampled from the high-frequency partition. COMET is not reported due to unsupported languages.
}
\end{table*}
\subsection{Prompting on Math Reasoning}
Figure \ref{fig:gsm8k} presents the overall accuracy of TFPD on the task of math reasoning with prompting experiments. Our proposed framework is effective on all models that we experimented on. On DeepSeek-V3, the accuracy goes from 63.55\% to 71.54\%. On GPT-4o-mini, the accuracy goes from 60.70\% to 68.70\%. On LlaMA3.3-70B-Instruct, it goes from 80.49\% to 88.75\%. We also conduct deeper analyses. Specifically, we calculate the intersection of low-frequency and high-frequency partitions. We found that when a sample pair has a correct model generation on its low-frequency partition, its high-frequency version is still correct. In other words, using our proposed framework only improved those samples which were originally answered incorrectly by the models on the low-frequency partition. For those ones which were originally answered correctly by the models on the low frequency partition, their performance is maintained with the high frequency partition.
\par
For space reasons, Table \ref{scaling} in the Appendix represents that our method is consistently useful and high-frequency data brings improvements on different sizes of qwen-2.5 models across 0.5b to 72b on the task of MR.
\par
Table \ref{chain-of-thought-evaluation} indicates that the chain-of-thought process is improved, which can be the reason why the math reasoning capabilities are improved.

\subsection{Prompting on Neural Machine Translation}
Figure \ref{fig:trans} demonstrates the results on Neural Machine Translation (NMT) on our TFPD dataset. The orange line indicates the model using the high-frequency partition in our TFPD dataset on ChatGPT or DeepSeek models. The results follow our proposed TFL, which suggests that high-frequency rephrases should be preferred as inputs into LLMs. Specifically, for all six results on all metrics we report and all baselines we conduct, high-frequency partition gives the best results in overall. We also found that ChatGPT and DeepSeek models are close in their translation results on the language pairs we conducted experiments on, as their Figure seems to be relatively similar to each other. This is reasonable, as both of them are strong LLMs. We also report results on 37 languages supported by the COMET model in use. The results also suggest the effectiveness of our proposed law. 
\par
Table \ref{summmt} summarises the improvements on NMT. We can see that when compared to our best baseline using the low-frequency partition, translation on most of the language pairs is improved. For example, 99 out of 100 language pairs are improved for BLEU on DeepSeek-V3. 63 of them are improved by more than 1 point. 31 of them are improved by more than 3 points, and 12 of them are improved by more than 5 points. The observations are consistent across all metrics, namely, BLEU, chrF, and COMET scores we use, across both DeepSeek-V3 and GPT-4o-mini, which suggests the effectiveness of our proposed law. When there is any performance degradation, they are all less than 1 point across the metrics and the models, which enhances our claim and the usefulness of our law.

\subsection{Prompting on Commonsense Reasoning} Table \ref{crresults} reports addtional results on the commonse reasoning partition CR. It clearly shows that the high-frequency part surpasses the low-frequency part. This validates the effectiveness of our method.
\subsection{Fine-tuning on Neural Machine Translation}
Table \ref{finetuning} presents our results for fine-tuning on NMT. There are three takeaways from this Table. 
\paragraph{High-frequency partition is even better than the ground-truth data} For the baseline of \textit{Fine-tuned Model}, \textit{FT on HF w/o TFD w/o CTFT} is even better across the languages and the metrics. The former one uses the original FLORES-200 dataset for fine-tuning, and the latter uses our TFPD dataset for fine-tuning without any TFD or CTFT. The improvements are obvious, for example, it improves from 4.6772 (+0\%) in BLEU to 5.2466 (+12.17\%) in BLEU on kea\_Latn.
\par 
\paragraph{High-frequency partition is better than the low-frequency partition} 
By looking at the baselines \textit{FT on LF w/o CTFT} and \textit{FT on HF w/o CTFT}. It is first clear that the latter one, using the high-frequency partition, is better than the former one, using a low-frequency partition. Interestingly, replacing half of the low-frequency partition randomly using the high-frequency partition can still obviously improve the results. Specifically, the improvement can be from 3.9073 (+0\%) to 4.4291 (+13.35\%) in BLEU on pag\_Latn.
\par
\paragraph{CTFT is useful for fine-tuning on translation} By looking at the baseline \textit{FT on HF w/o CTFT} and \textit{FT on HF w/ CTFT}, the latter one trains the model using CTFT, from the order of low-to-high in terms of the textual frequency. This yields 8/8 of the best metrics we got in all the experiments. Specifically, the improvement can be from 3.7781 (+0\%) to 4.9102 (+29.96\%) in BLEU on pag\_Latn.
\begin{table*}[t!]
\small
\centering
    \setlength\tabcolsep{5pt}
    \setlength\extrarowheight{0pt}
\begin{tabular}{l|ccccc}
\hline
 
\textbf{Metric} & \textbf{High-Frequency} & \textbf{Low-Frequency} & \textbf{$\Delta$(HF-LF)}
& \textbf{Pearson Corr.} & \textbf{Spearman Corr.}\\
\hline
\multicolumn{6}{c}{\textit{Math Reasoning}}\\
\hline
Max Dependency Tree Depth & 5.02&	5.72&	-0.70& -0.0447& -0.0285\\
Mean Dependency Distance & 2.12	&2.22	&-0.10 &-0.0086 & 0.0094 \\
Flesch-Kincaid Grade Level & 4.36& 6.35	& -1.99&-0.0799 & -0.0545 
\\
\hline
\multicolumn{6}{c}{\textit{Machine Translation}}\\
\hline
Max Dependency Tree Depth & 5.52 &7.51	&-1.99  &-0.2713 &-0.2822\\
Mean Dependency Distance & 2.31&	2.47&	-0.16 &-0.1137 &-0.1257\\
Flesch-Kincaid Grade Level & 8.97	&9.08	&-0.11 &-0.1673 & -0.1528 
\\
\hline
\end{tabular}
\caption{\label{complexity}
Textual complexity metrics and their correlation with frequency. Corr. denotes correlation. We use nlp = spacy.load("en\_core\_web\_sm") for calculation.
}
\end{table*}

\begin{table*}[t!]
\tiny
\centering
    \setlength\tabcolsep{13pt}
    \setlength\extrarowheight{0pt}
\begin{tabular}{l|ccccccc}
\hline
\textbf{Bin Range} & \textbf{N} & \textbf{BLEU(HF)} & \textbf{BLEU(LF)} & \textbf{$\Delta$BLEU(HF-LF)} & \textbf{chrF(HF)} & \textbf{chrF(LF)} & \textbf{$\Delta$chrF(HF-LF)} \\
\hline
Strict Depth Match & 144 & 20.82 & 16.04 & +4.78 & 48.73 & 43.86 & +4.87 \\
$[0\%, 5\%)$ & 144 & 20.82 & 16.04 & +4.78 & 48.73 & 43.86 & +4.87 \\
$[5\%, 10\%)$ & 6 & 22.45 & 14.79 & +7.65 & 49.76 & 49.19 & +0.57 \\
$[10\%, 15\%)$ & 71 & 19.12 & 15.38 & +3.74 & 46.19 & 44.71 & +1.47 \\
$[15\%, 20\%)$ & 65 & 20.93 & 14.77 & +6.16 & 48.91 & 43.46 & +5.45 \\
$[20\%, 25\%)$ & 53 & 24.08 & 18.52 & +5.56 & 50.87 & 44.27 & +6.60 \\
$[25\%, 30\%)$ & 65 & 19.75 & 12.54 & +7.21 & 47.53 & 42.51 & +5.01 \\
$[30\%, 35\%)$ & 41 & 19.90 & 12.61 & +7.29 & 47.78 & 43.72 & +4.05 \\
$[35\%, 40\%)$ & 17 & 19.03 & 14.13 & +4.90 & 44.22 & 42.62 & +1.60 \\
$[40\%, 45\%)$ & 28 & 16.53 & 9.76 & +6.77 & 46.47 & 40.92 & +5.55 \\
$[50\%, 55\%)$ & 21 & 13.89 & 16.20 & -2.31 & 41.65 & 46.86 & -5.21 \\
$[55\%, 60\%)$ & 9 & 10.93 & 3.33 & +7.60 & 45.62 & 38.92 & +6.70 \\
$[60\%, 65\%)$ & 4 & 17.13 & 12.18 & +4.95 & 43.30 & 44.46 & -1.16 \\
$[65\%, 70\%)$ & 2 & 15.54 & 4.36 & +11.17 & 37.17 & 39.72 & -2.56 \\
\hline
\end{tabular}
\caption{\label{complexity2}
Separated bins with those high-frequency and low-frequency samples with restricted tree depth difference.
}
\end{table*}
\subsection{Analysis on TFD}

\begin{figure}[ht!]
\begin{center}
\vspace{0mm}
\centerline{
\includegraphics[width=7.5cm]{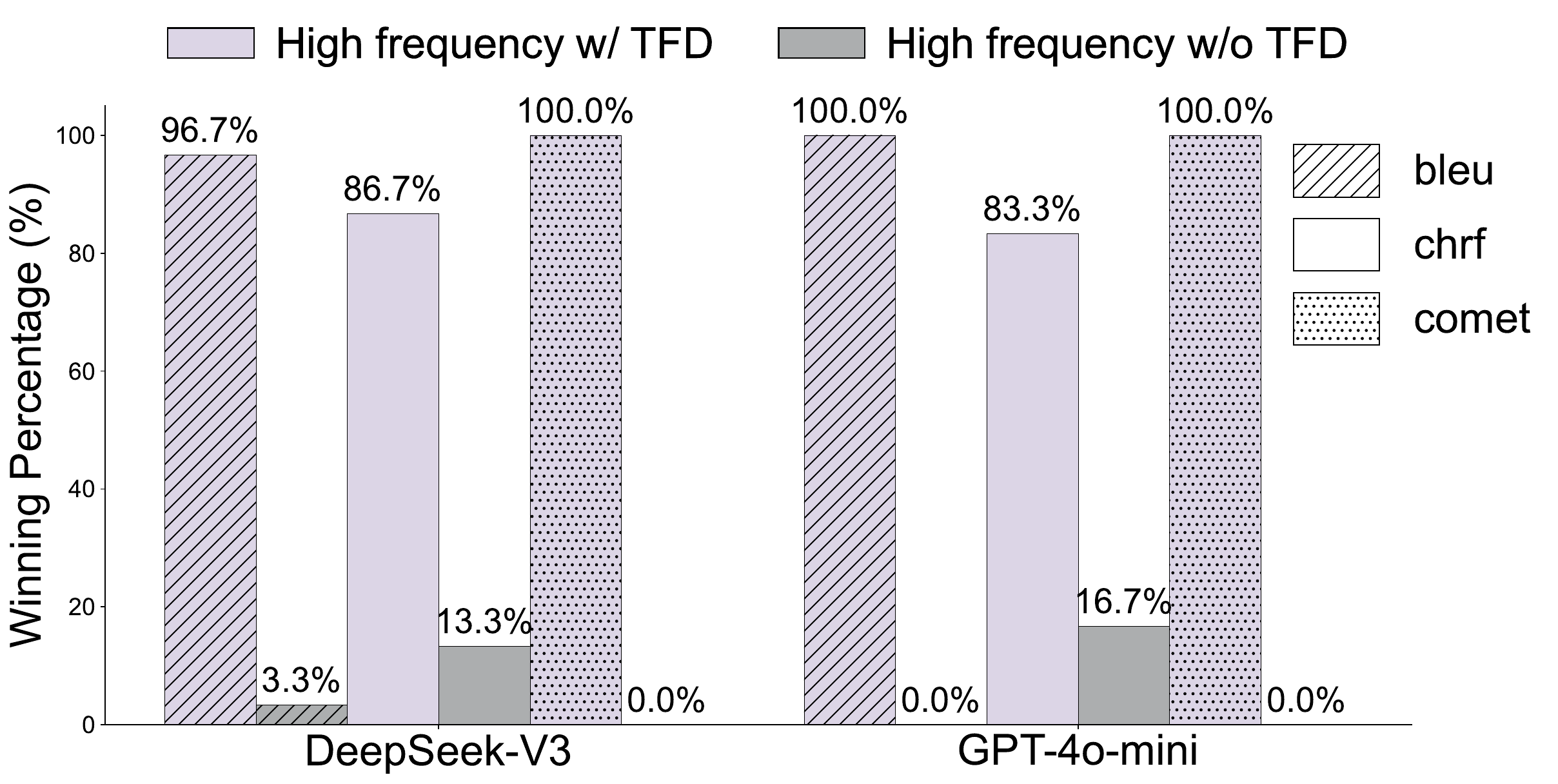}}
    \caption{The ablation study results of TFD on TFPD. The results are compared on BLEU, chrf and COMET. The bars are plotted in terms of the winning percentages.}
    \label{fig:abstfd}
\end{center}
\vspace{-5mm}
\end{figure}

Figure \ref{fig:abstfd} presents the ablation study on TFD. It is obvious that removing TFD causes a drop in performance. For example, 100\% of the language pairs are better with TFD on COMET scores with DeepSeek-V3. This validates the usefulness of TFD. Figure \ref{fig:tfd2} also demonstrates the relationship between the amount of data used for frequency distillation and the performance improvement. Overall, with more data used for TFD, there is a greater performance gain. This further validates the usefulness of TFD. Finally, combining prompting with higher-frequency paraphrases on models with CTFT as a whole framework is useful, as presented in the Appendix in Table \ref{finetuningplus}.

\subsection{Correlation on Frequency} For space reasons, we present a correlation analysis between textual frequency and final translation performance, even when the instances are not paraphrases to each other using the full translation dataset in TFPD. We present the final results in Appendix in Table \ref{pptable}. There is are strong correlation (1.0) on multiple languages when translating from English. This strengths our claim.
\par 
Table \ref{complexity} represents the relationship between textual complexity and frequency, and we see that they have very weak correlation. This enhances the usefulness of our method by distingushing TFL from the traditional curriculum learning. Table \ref{complexity2} shows that in most bins, high-frequency prompts are better. Only in 1 bin [50\%-55\%], the low-frequency prompts are better on BLEU and chrF, but there are only 21 samples in this bin. This means that high-frequency prompts are consistently better.
\par 
Finally, we present a theoretical proof in Appendix to strength our claim.

\section{Conclusions}
This paper proposed a framework for textual frequency on LLMs, which is composed of three units, namely TFL, TFD, and CTFT. High-frequency inputs are suggested by our framework, in both tuning and training on LLMs, which can be combined with curriculum learning to improve final performance. We conduct experiments on tasks of Math Reasoning, Machine Translation on hundreds of language pairs, Commonsense Reasoning, and Agentic Tool Calling. Experimental results and extensive analysis suggest the effectiveness of our textual frequency framework. Extensive analysis indicates that when inputs are even different, the final outputs of LLMs are positively related to textual frequency, which further suggests the soundness of our proposed framework.

\section*{Limitations}
Using story completion to obtain frequency estimation can bring certain computational costs, yet this workaround the necessity of obtaining closed-resourced training corpora of LLMs, which is often unrealistic.
\section*{Ethical Statement}
We honour and support the ACL ARR Code of Ethics. The datasets used in this work are well-known and widely used, and the dataset pre-processing does not make use of any external textual resource. In our view, there is no known ethical issue. End-to-end pre-trained LLMs are also used, which are subjected to generating offensive context. But the above-mentioned issues are widely known to commonly exist for these models. Any content generated do not reflect the view of the authors.

\bibliography{custom}
\newpage
\appendix

\section*{Appendix}

\begin{table}[htb]
\centering
\setlength\tabcolsep{10pt}
\setlength\aboverulesep{0pt}\setlength\belowrulesep{0pt}
\setcellgapes{3pt}\makegapedcells
\begin{tabular}{c|c|c}
\hline
\multicolumn{3}{c}{\textit{Supported Languages by COMET}}\\
\hline
ell\_Grek &  spa\_Latn & bel\_Cyrl \\
acm\_Arab &  hrv\_Latn & mar\_Deva \\
srp\_Cyrl &  uig\_Arab & est\_Latn \\
pol\_Latn & ukr\_Cyrl & eus\_Latn \\ 
ajp\_Arab & mkd\_Cyrl & swe\_Latn \\
urd\_Arab & ind\_Latn & swh\_Latn\\
uzn\_Latn & fin\_Latn & ita\_Latn \\
kor\_Hang & lao\_Laoo & rus\_Cyrl \\
arb\_Arab & bul\_Cyrl & nld\_Latn \\
san\_Deva & ars\_Arab & lit\_Latn \\
tha\_Thai & glg\_Latn & slk\_Latn \\
cym\_Latn & dan\_Latn & snd\_Arab \\
som\_Latn& - & -\\
\hline
\end{tabular}
\caption{\label{np}
The list of 37 languages supported by our COMET model for evaluation on machine translation.
}
\end{table}
\begin{table*}[ht]  
\centering  
\small  
\begin{adjustbox}{max width=\textwidth}  
    \begin{tabular}{L|CC|L|CC|L|CC|L|CC|L|CC}  
        \toprule  
        \textbf{Language} & \textbf{Low} & \textbf{High} &   
        \textbf{Language} & \textbf{Low} & \textbf{High} &   
        \textbf{Language} & \textbf{Low} & \textbf{High} &   
        \textbf{Language} & \textbf{Low} & \textbf{High} &   
        \textbf{Language} & \textbf{Low} & \textbf{High} \\
        \midrule  

acm\_Arab&2.54&\hl{\textbf{3.29}} & acq\_Arab&3.07&\hl{\textbf{4.51}} & aeb\_Arab&2.36&\hl{\textbf{3.22}} & ajp\_Arab&2.7&\hl{\textbf{4.14}} & als\_Latn&9.36&\hl{\textbf{14.54}} \\
arb\_Arab&6.15&\hl{\textbf{9.24}} & ars\_Arab&5.2&\hl{\textbf{6.36}} & ary\_Arab&0.73&\hl{\textbf{0.98}} & arz\_Arab&3.17&\hl{\textbf{4.66}} & awa\_Deva&1.93&\hl{\textbf{3.06}} \\
ayr\_Latn&0.41&\hl{\textbf{0.59}} & ban\_Latn&3.7&\hl{\textbf{4.24}} & bel\_Cyrl&3.7&\hl{\textbf{5.66}} & bho\_Deva&3.38&\hl{\textbf{4.43}} & bjn\_Latn&5.31&\hl{\textbf{6.08}} \\
bul\_Cyrl&10.61&\hl{\textbf{16.97}} & ceb\_Latn&13.06&\hl{\textbf{17.26}} & ckb\_Arab&1.54&\hl{\textbf{2.47}} & crh\_Latn&1.67&\hl{\textbf{2.32}} & cym\_Latn&15.06&\hl{\textbf{20.64}} \\
dan\_Latn&12.58&\hl{\textbf{21.04}} & dzo\_Tibt&0.03&\hl{\textbf{0.04}} & ell\_Grek&7.97&\hl{\textbf{12.07}} & est\_Latn&6.1&\hl{\textbf{9.81}} & eus\_Latn&3.94&\hl{\textbf{6.09}} \\
ewe\_Latn&1.18&\hl{\textbf{1.54}} & fin\_Latn&5.42&\hl{\textbf{9.44}} & fon\_Latn&0.26&\hl{\textbf{0.39}} & glg\_Latn&11.01&\hl{\textbf{16.84}} & grn\_Latn&1.6&\hl{\textbf{2.0}} \\
guj\_Gujr&4.17&\hl{\textbf{6.97}} & hne\_Deva&2.08&\hl{\textbf{2.32}} & hrv\_Latn&8.74&\hl{\textbf{13.83}} & ilo\_Latn&7.99&\hl{\textbf{10.89}} & ind\_Latn&13.91&\hl{\textbf{20.26}} \\
ita\_Latn&9.86&\hl{\textbf{15.95}} & kab\_Latn&1.03&\hl{\textbf{1.21}} & kac\_Latn&1.36&\hl{\textbf{1.72}} & kan\_Knda&2.71&\hl{\textbf{4.57}} & kas\_Arab&0.4&\hl{\textbf{0.55}} \\
kas\_Deva&0.16&\hl{\textbf{0.32}} & kat\_Geor&2.77&\hl{\textbf{4.72}} & kea\_Latn&3.84&\hl{\textbf{5.58}} & kmr\_Latn&2.45&\hl{\textbf{3.03}} & kon\_Latn&3.16&\hl{\textbf{4.43}} \\
kor\_Hang&2.95&\hl{\textbf{5.02}} & lao\_Laoo&1.46&\hl{\textbf{1.73}} & lin\_Latn&3.83&\hl{\textbf{5.06}} & lit\_Latn&6.62&\hl{\textbf{10.15}} & lmo\_Latn&2.99&\hl{\textbf{3.77}} \\
ltz\_Latn&6.69&\hl{\textbf{10.03}} & lug\_Latn&1.47&\hl{\textbf{2.04}} & luo\_Latn&1.04&\hl{\textbf{1.4}} & lus\_Latn&2.78&\hl{\textbf{3.26}} & mag\_Deva&4.2&\hl{\textbf{5.24}} \\
mai\_Deva&2.97&\hl{\textbf{3.97}} & mal\_Mlym&2.12&\hl{\textbf{3.22}} & mar\_Deva&3.04&\hl{\textbf{5.03}} & min\_Latn&6.05&\hl{\textbf{7.8}} & mkd\_Cyrl&10.06&\hl{\textbf{14.87}} \\
mlt\_Latn&7.73&\hl{\textbf{11.17}} & mya\_Mymr&0.54&\hl{\textbf{0.66}} & nld\_Latn&8.65&\hl{\textbf{12.27}} & nno\_Latn&8.55&\hl{\textbf{14.01}} & nob\_Latn&8.54&\hl{\textbf{12.84}} \\
pbt\_Arab&2.29&\hl{\textbf{3.08}} & pol\_Latn&6.74&\hl{\textbf{11.03}} & prs\_Arab&5.27&\hl{\textbf{7.65}} & quy\_Latn&0.55&\hl{\textbf{0.68}} & run\_Latn&1.47&\hl{\textbf{2.05}} \\
rus\_Cyrl&9.42&\hl{\textbf{16.08}} & sag\_Latn&0.67&\hl{\textbf{0.9}} & san\_Deva&0.11&\hl{\textbf{0.38}} & sat\_Olck&1.0&\hl{\textbf{1.71}} & scn\_Latn&4.08&\hl{\textbf{5.94}} \\
sin\_Sinh&2.46&\hl{\textbf{3.96}} & slk\_Latn&8.49&\hl{\textbf{13.31}} & sna\_Latn&2.94&\hl{\textbf{4.54}} & snd\_Arab&4.35&\hl{\textbf{6.15}} & som\_Latn&2.5&\hl{\textbf{3.25}} \\
spa\_Latn&10.36&\hl{\textbf{14.89}} & srd\_Latn&7.44&\hl{\textbf{11.66}} & srp\_Cyrl&8.61&\hl{\textbf{14.29}} & ssw\_Latn&1.25&\hl{\textbf{1.53}} & sun\_Latn&5.6&\hl{\textbf{7.42}} \\
swe\_Latn&11.7&\hl{\textbf{18.63}} & swh\_Latn&10.76&\hl{\textbf{15.63}} & szl\_Latn&3.45&\hl{\textbf{5.02}} & tat\_Cyrl&3.75&\hl{\textbf{5.86}} & tgk\_Cyrl&4.5&\hl{\textbf{6.25}} \\
tgl\_Latn&14.73&\hl{\textbf{19.68}} & tha\_Thai&0.95&\hl{\textbf{1.3}} & tpi\_Latn&8.8&\hl{\textbf{10.85}} & twi\_Latn&2.44&\hl{\textbf{3.08}} & uig\_Arab&1.5&\hl{\textbf{2.27}} \\
ukr\_Cyrl&7.8&\hl{\textbf{12.75}} & urd\_Arab&6.26&\hl{\textbf{9.62}} & uzn\_Latn&4.1&\hl{\textbf{5.98}} & war\_Latn&10.42&\hl{\textbf{13.29}} & zho\_Hans & \hl{\textbf{0.42}}&0.26 \\

        \bottomrule  
    \end{tabular}  
\end{adjustbox}  
\caption{Results on DEEPSEEK-V3 in BLEU scores on 100 languages from English into other languages.}  
\label{tab:language_metrics1}  
\end{table*}

\begin{table*}[ht]  
\centering  
\small  
\begin{adjustbox}{max width=\textwidth}  
    \begin{tabular}{L|CC|L|CC|L|CC|L|CC|L|CC}  
        \toprule  
        \textbf{Language} & \textbf{Low} & \textbf{High} &   
        \textbf{Language} & \textbf{Low} & \textbf{High} &   
        \textbf{Language} & \textbf{Low} & \textbf{High} &   
        \textbf{Language} & \textbf{Low} & \textbf{High} &   
        \textbf{Language} & \textbf{Low} & \textbf{High} \\
        \midrule  

acm\_Arab&36.94&\hl{\textbf{38.42}} & acq\_Arab&37.07&\hl{\textbf{39.17}} & aeb\_Arab&34.44&\hl{\textbf{36.79}} & ajp\_Arab&37.79&\hl{\textbf{40.58}} & als\_Latn&42.09&\hl{\textbf{45.98}} \\
arb\_Arab&40.87&\hl{\textbf{45.4}} & ars\_Arab&39.43&\hl{\textbf{41.77}} & ary\_Arab&30.51&\hl{\textbf{31.48}} & arz\_Arab&36.87&\hl{\textbf{39.65}} & awa\_Deva&30.45&\hl{\textbf{31.84}} \\
ayr\_Latn&28.95&\hl{\textbf{30.13}} & ban\_Latn&36.79&\hl{\textbf{38.27}} & bel\_Cyrl&35.88&\hl{\textbf{38.77}} & bho\_Deva&31.92&\hl{\textbf{33.93}} & bjn\_Latn&41.31&\hl{\textbf{42.63}} \\
bul\_Cyrl&45.11&\hl{\textbf{50.17}} & ceb\_Latn&46.96&\hl{\textbf{49.72}} & ckb\_Arab&38.68&\hl{\textbf{40.8}} & crh\_Latn&34.55&\hl{\textbf{36.55}} & cym\_Latn&45.09&\hl{\textbf{49.54}} \\
dan\_Latn&45.89&\hl{\textbf{51.74}} & dzo\_Tibt&33.62&\hl{\textbf{33.99}} & ell\_Grek&39.16&\hl{\textbf{43.04}} & est\_Latn&44.16&\hl{\textbf{47.71}} & eus\_Latn&44.93&\hl{\textbf{48.09}} \\
ewe\_Latn&27.1&\hl{\textbf{27.87}} & fin\_Latn&44.2&\hl{\textbf{48.79}} & fon\_Latn&16.56&\hl{\textbf{17.63}} & glg\_Latn&43.55&\hl{\textbf{48.02}} & grn\_Latn&30.24&\hl{\textbf{31.04}} \\
guj\_Gujr&37.1&\hl{\textbf{40.45}} & hne\_Deva&29.64&\hl{\textbf{30.44}} & hrv\_Latn&42.99&\hl{\textbf{48.09}} & ilo\_Latn&43.98&\hl{\textbf{46.68}} & ind\_Latn&50.07&\hl{\textbf{55.44}} \\
ita\_Latn&44.39&\hl{\textbf{48.37}} & kab\_Latn&26.12&\hl{\textbf{27.1}} & kac\_Latn&27.85&\hl{\textbf{29.04}} & kan\_Knda&40.54&\hl{\textbf{44.1}} & kas\_Arab&22.95&\hl{\textbf{23.67}} \\
kas\_Deva&17.15&\hl{\textbf{17.88}} & kat\_Geor&41.81&\hl{\textbf{44.9}} & kea\_Latn&35.78&\hl{\textbf{37.61}} & kmr\_Latn&33.7&\hl{\textbf{35.35}} & kon\_Latn&36.24&\hl{\textbf{37.43}} \\
kor\_Hang&25.15&\hl{\textbf{29.23}} & lao\_Laoo&38.13&\hl{\textbf{39.98}} & lin\_Latn&38.71&\hl{\textbf{40.32}} & lit\_Latn&42.62&\hl{\textbf{47.39}} & lmo\_Latn&30.0&\hl{\textbf{31.49}} \\
ltz\_Latn&41.44&\hl{\textbf{45.1}} & lug\_Latn&34.0&\hl{\textbf{35.39}} & luo\_Latn&27.14&\hl{\textbf{27.33}} & lus\_Latn&33.3&\hl{\textbf{34.14}} & mag\_Deva&32.78&\hl{\textbf{33.65}} \\
mai\_Deva&34.31&\hl{\textbf{36.05}} & mal\_Mlym&40.38&\hl{\textbf{43.88}} & mar\_Deva&38.07&\hl{\textbf{40.93}} & min\_Latn&41.89&\hl{\textbf{44.06}} & mkd\_Cyrl&44.65&\hl{\textbf{48.62}} \\
mlt\_Latn&42.62&\hl{\textbf{47.33}} & mya\_Mymr&42.56&\hl{\textbf{44.17}} & nld\_Latn&44.19&\hl{\textbf{47.9}} & nno\_Latn&42.43&\hl{\textbf{46.45}} & nob\_Latn&43.35&\hl{\textbf{46.56}} \\
pbt\_Arab&28.45&\hl{\textbf{30.04}} & pol\_Latn&40.63&\hl{\textbf{44.69}} & prs\_Arab&36.75&\hl{\textbf{39.93}} & quy\_Latn&35.37&\hl{\textbf{35.86}} & run\_Latn&31.63&\hl{\textbf{33.21}} \\
rus\_Cyrl&43.15&\hl{\textbf{47.97}} & sag\_Latn&20.36&\hl{\textbf{21.52}} & san\_Deva&27.9&\hl{\textbf{29.61}} & sat\_Olck&28.63&\hl{\textbf{29.95}} & scn\_Latn&36.77&\hl{\textbf{39.9}} \\
sin\_Sinh&35.75&\hl{\textbf{38.05}} & slk\_Latn&40.81&\hl{\textbf{45.41}} & sna\_Latn&41.51&\hl{\textbf{43.93}} & snd\_Arab&33.16&\hl{\textbf{36.0}} & som\_Latn&36.94&\hl{\textbf{38.57}} \\
spa\_Latn&42.38&\hl{\textbf{46.18}} & srd\_Latn&40.51&\hl{\textbf{43.96}} & srp\_Cyrl&41.7&\hl{\textbf{47.09}} & ssw\_Latn&36.89&\hl{\textbf{38.53}} & sun\_Latn&41.18&\hl{\textbf{44.26}} \\
swe\_Latn&45.86&\hl{\textbf{51.3}} & swh\_Latn&47.58&\hl{\textbf{50.84}} & szl\_Latn&34.92&\hl{\textbf{36.88}} & tat\_Cyrl&39.43&\hl{\textbf{42.32}} & tgk\_Cyrl&38.33&\hl{\textbf{40.6}} \\
tgl\_Latn&49.2&\hl{\textbf{51.62}} & tha\_Thai&43.24&\hl{\textbf{47.0}} & tpi\_Latn&40.08&\hl{\textbf{40.7}} & twi\_Latn&31.6&\hl{\textbf{32.54}} & uig\_Arab&37.42&\hl{\textbf{39.42}} \\
ukr\_Cyrl&41.18&\hl{\textbf{45.68}} & urd\_Arab&37.15&\hl{\textbf{40.79}} & uzn\_Latn&44.17&\hl{\textbf{46.69}} & war\_Latn&44.35&\hl{\textbf{46.44}} & zho\_Hans&24.44&\hl{\textbf{30.46}} \\

        \bottomrule  
    \end{tabular}  
\end{adjustbox}  
\caption{Results on DEEPSEEK-V3 in chrF scores on 100 languages from English into other languages.}  
\label{tab:language_metrics2}  
\end{table*}  

\begin{table*}[ht]  
\centering  
\small  
\begin{adjustbox}{max width=\textwidth}  
    \begin{tabular}{L|CC|L|CC|L|CC|L|CC|L|CC}  
        \toprule  
        \textbf{Language} & \textbf{Low} & \textbf{High} &   
        \textbf{Language} & \textbf{Low} & \textbf{High} &   
        \textbf{Language} & \textbf{Low} & \textbf{High} &   
        \textbf{Language} & \textbf{Low} & \textbf{High} &   
        \textbf{Language} & \textbf{Low} & \textbf{High} \\
        \midrule  

acm\_Arab&79.82&\hl{\textbf{80.02}} & ajp\_Arab&78.76&\hl{\textbf{79.16}} & arb\_Arab&82.95&\hl{\textbf{85.1}} & ars\_Arab&82.48&\hl{\textbf{83.41}} & bel\_Cyrl&83.02&\hl{\textbf{84.84}} \\
bul\_Cyrl&85.82&\hl{\textbf{88.27}} & cym\_Latn&79.78&\hl{\textbf{82.84}} & dan\_Latn&84.29&\hl{\textbf{87.59}} & ell\_Grek&84.76&\hl{\textbf{86.84}} & est\_Latn&87.09&\hl{\textbf{89.11}} \\
eus\_Latn&82.08&\hl{\textbf{84.12}} & fin\_Latn&87.54&\hl{\textbf{90.23}} & glg\_Latn&81.05&\hl{\textbf{83.64}} & hrv\_Latn&86.29&\hl{\textbf{88.58}} & ind\_Latn&86.74&\hl{\textbf{89.0}} \\
ita\_Latn&82.92&\hl{\textbf{85.22}} & kor\_Hang&86.78&\hl{\textbf{88.26}} & lao\_Laoo&80.56&\hl{\textbf{81.85}} & lit\_Latn&85.09&\hl{\textbf{88.26}} & mar\_Deva&69.13&\hl{\textbf{71.45}} \\
mkd\_Cyrl&83.87&\hl{\textbf{86.4}} & nld\_Latn&82.36&\hl{\textbf{85.33}} & pol\_Latn&85.3&\hl{\textbf{87.72}} & rus\_Cyrl&85.49&\hl{\textbf{87.73}} & san\_Deva&70.71&\hl{\textbf{71.68}} \\
slk\_Latn&84.98&\hl{\textbf{87.64}} & snd\_Arab&73.78&\hl{\textbf{75.95}} & som\_Latn&75.33&\hl{\textbf{76.9}} & spa\_Latn&80.78&\hl{\textbf{83.2}} & srp\_Cyrl&84.27&\hl{\textbf{86.98}} \\
swe\_Latn&84.61&\hl{\textbf{87.69}} & swh\_Latn&79.48&\hl{\textbf{81.43}} & tha\_Thai&85.22&\hl{\textbf{86.79}} & uig\_Arab&80.28&\hl{\textbf{81.85}} & ukr\_Cyrl&85.55&\hl{\textbf{87.98}} \\
urd\_Arab&78.62&\hl{\textbf{80.58}} & uzn\_Latn&86.76&\hl{\textbf{88.09}} \\
 
        \bottomrule  
    \end{tabular}  
\end{adjustbox}  
\caption{Results on DEEPSEEK-V3 in COMET scores on 37 supported languages from English into other languages.}  
\label{tab:language_metrics3}  
\end{table*}

\begin{table*}[ht]  
\centering  
\small  
\begin{adjustbox}{max width=\textwidth}  
    \begin{tabular}{L|CC|L|CC|L|CC|L|CC|L|CC}  
        \toprule  
        \textbf{Language} & \textbf{Low} & \textbf{High} &   
        \textbf{Language} & \textbf{Low} & \textbf{High} &   
        \textbf{Language} & \textbf{Low} & \textbf{High} &   
        \textbf{Language} & \textbf{Low} & \textbf{High} &   
        \textbf{Language} & \textbf{Low} & \textbf{High} \\
        \midrule  

acm\_Arab&2.83&\hl{\textbf{3.79}} & acq\_Arab&3.41&\hl{\textbf{4.43}} & aeb\_Arab&2.12&\hl{\textbf{3.19}} & ajp\_Arab&3.14&\hl{\textbf{4.28}} & als\_Latn&8.96&\hl{\textbf{12.97}} \\
arb\_Arab&5.77&\hl{\textbf{8.92}} & ars\_Arab&4.76&\hl{\textbf{6.17}} & ary\_Arab&1.27&\hl{\textbf{1.84}} & arz\_Arab&2.84&\hl{\textbf{4.6}} & awa\_Deva&2.23&\hl{\textbf{3.03}} \\
ayr\_Latn&0.52&\hl{\textbf{0.67}} & ban\_Latn&2.7&\hl{\textbf{3.47}} & bel\_Cyrl&3.07&\hl{\textbf{4.6}} & bho\_Deva&2.82&\hl{\textbf{4.06}} & bjn\_Latn&2.71&\hl{\textbf{3.12}} \\
bul\_Cyrl&10.5&\hl{\textbf{16.51}} & ceb\_Latn&12.3&\hl{\textbf{16.0}} & ckb\_Arab&0.2&\hl{\textbf{0.48}} & crh\_Latn&0.8&\hl{\textbf{1.08}} & cym\_Latn&12.34&\hl{\textbf{16.61}} \\
dan\_Latn&13.1&\hl{\textbf{19.69}} & dzo\_Tibt&0.08&\hl{\textbf{0.08}} & ell\_Grek&7.42&\hl{\textbf{11.25}} & est\_Latn&6.16&\hl{\textbf{9.59}} & eus\_Latn&2.49&\hl{\textbf{4.58}} \\
ewe\_Latn&0.61&\hl{\textbf{0.67}} & fin\_Latn&5.24&\hl{\textbf{8.74}} & fon\_Latn&0.26&\hl{\textbf{0.4}} & glg\_Latn&10.71&\hl{\textbf{15.36}} & grn\_Latn&1.07&\hl{\textbf{1.28}} \\
guj\_Gujr&2.58&\hl{\textbf{3.55}} & hne\_Deva&2.36&\hl{\textbf{3.3}} & hrv\_Latn&7.88&\hl{\textbf{12.22}} & ilo\_Latn&5.75&\hl{\textbf{7.02}} & ind\_Latn&13.96&\hl{\textbf{19.42}} \\
ita\_Latn&9.64&\hl{\textbf{14.48}} & kab\_Latn&0.36&\hl{\textbf{0.45}} & kac\_Latn&0.39&\hl{\textbf{0.46}} & kan\_Knda&1.39&\hl{\textbf{1.84}} & kas\_Arab&0.3&\hl{\textbf{0.35}} \\
kas\_Deva&0.09&\hl{\textbf{0.18}} & kat\_Geor&2.23&\hl{\textbf{2.85}} & kea\_Latn&2.54&\hl{\textbf{2.83}} & kmr\_Latn&1.27&\hl{\textbf{1.67}} & kon\_Latn&1.35&\hl{\textbf{1.44}} \\
kor\_Hang&3.45&\hl{\textbf{4.99}} & lao\_Laoo & \hl{\textbf{1.07}}&0.8 & lin\_Latn&2.77&\hl{\textbf{3.43}} & lit\_Latn&5.93&\hl{\textbf{9.37}} & lmo\_Latn&1.75&\hl{\textbf{2.06}} \\
ltz\_Latn&4.46&\hl{\textbf{6.0}} & lug\_Latn&1.3&\hl{\textbf{1.58}} & luo\_Latn&0.98&\hl{\textbf{1.25}} & lus\_Latn&2.1&\hl{\textbf{2.19}} & mag\_Deva&3.97&\hl{\textbf{5.16}} \\
mai\_Deva&3.09&\hl{\textbf{3.75}} & mal\_Mlym&0.64&\hl{\textbf{1.0}} & mar\_Deva&2.45&\hl{\textbf{3.19}} & min\_Latn&2.94&\hl{\textbf{3.51}} & mkd\_Cyrl&9.08&\hl{\textbf{12.63}} \\
mlt\_Latn&5.7&\hl{\textbf{8.11}} & mya\_Mymr&0.21&\hl{\textbf{0.33}} & nld\_Latn&8.38&\hl{\textbf{11.61}} & nno\_Latn&8.7&\hl{\textbf{13.66}} & nob\_Latn&8.77&\hl{\textbf{13.53}} \\
pbt\_Arab&2.14&\hl{\textbf{3.01}} & pol\_Latn&6.26&\hl{\textbf{10.36}} & prs\_Arab&5.24&\hl{\textbf{7.12}} & quy\_Latn & \hl{\textbf{0.53}}&0.51 & run\_Latn&1.84&\hl{\textbf{2.0}} \\
rus\_Cyrl&8.82&\hl{\textbf{14.08}} & sag\_Latn & \hl{\textbf{0.74}}&0.66 & san\_Deva & \hl{\textbf{0.22}}&0.2 & sat\_Olck&0.0&\hl{\textbf{0.15}} & scn\_Latn&3.05&\hl{\textbf{4.07}} \\
sin\_Sinh&0.66&\hl{\textbf{1.14}} & slk\_Latn&8.02&\hl{\textbf{12.21}} & sna\_Latn&2.28&\hl{\textbf{2.77}} & snd\_Arab&3.82&\hl{\textbf{5.65}} & som\_Latn&2.97&\hl{\textbf{3.94}} \\
spa\_Latn&10.31&\hl{\textbf{13.52}} & srd\_Latn&3.05&\hl{\textbf{3.67}} & srp\_Cyrl&7.26&\hl{\textbf{12.01}} & ssw\_Latn&0.78&\hl{\textbf{0.93}} & sun\_Latn&4.62&\hl{\textbf{6.5}} \\
swe\_Latn&11.57&\hl{\textbf{19.5}} & swh\_Latn&10.94&\hl{\textbf{14.41}} & szl\_Latn&2.34&\hl{\textbf{2.74}} & tat\_Cyrl&3.45&\hl{\textbf{5.06}} & tgk\_Cyrl&3.43&\hl{\textbf{4.92}} \\
tgl\_Latn&15.29&\hl{\textbf{19.16}} & tha\_Thai&1.24&\hl{\textbf{1.61}} & tpi\_Latn&5.91&\hl{\textbf{6.67}} & twi\_Latn&1.87&\hl{\textbf{2.06}} & uig\_Arab&0.31&\hl{\textbf{0.49}} \\
ukr\_Cyrl&7.91&\hl{\textbf{12.1}} & urd\_Arab&5.38&\hl{\textbf{8.28}} & uzn\_Latn&3.76&\hl{\textbf{4.95}} & war\_Latn&11.11&\hl{\textbf{13.62}} & zho\_Hans & \hl{\textbf{0.59}}&0.33 \\

        \bottomrule  
    \end{tabular}  
\end{adjustbox}  
\caption{Results on GPT4o-mini in BLEU scores on 100 languages from English into other languages.}  
\label{tab:language_metrics4}  
\end{table*}

\begin{table*}[ht]  
\centering  
\small  
\begin{adjustbox}{max width=\textwidth}  
    \begin{tabular}{L|CC|L|CC|L|CC|L|CC|L|CC}  
        \toprule  
        \textbf{Language} & \textbf{Low} & \textbf{High} &   
        \textbf{Language} & \textbf{Low} & \textbf{High} &   
        \textbf{Language} & \textbf{Low} & \textbf{High} &   
        \textbf{Language} & \textbf{Low} & \textbf{High} &   
        \textbf{Language} & \textbf{Low} & \textbf{High} \\
        \midrule  

acm\_Arab&36.79&\hl{\textbf{38.99}} & acq\_Arab&35.92&\hl{\textbf{38.38}} & aeb\_Arab&33.7&\hl{\textbf{36.16}} & ajp\_Arab&37.53&\hl{\textbf{40.66}} & als\_Latn&41.18&\hl{\textbf{44.51}} \\
arb\_Arab&39.4&\hl{\textbf{43.97}} & ars\_Arab&38.25&\hl{\textbf{40.32}} & ary\_Arab&32.15&\hl{\textbf{34.47}} & arz\_Arab&35.72&\hl{\textbf{38.95}} & awa\_Deva&29.79&\hl{\textbf{31.08}} \\
ayr\_Latn&25.05&\hl{\textbf{25.17}} & ban\_Latn&34.58&\hl{\textbf{36.19}} & bel\_Cyrl&34.16&\hl{\textbf{36.62}} & bho\_Deva&29.98&\hl{\textbf{32.09}} & bjn\_Latn&33.84&\hl{\textbf{35.02}} \\
bul\_Cyrl&43.87&\hl{\textbf{48.82}} & ceb\_Latn&45.81&\hl{\textbf{48.45}} & ckb\_Arab&27.2&\hl{\textbf{28.15}} & crh\_Latn&27.31&\hl{\textbf{28.67}} & cym\_Latn&41.43&\hl{\textbf{45.22}} \\
dan\_Latn&45.5&\hl{\textbf{50.65}} & dzo\_Tibt&22.06&\hl{\textbf{22.24}} & ell\_Grek&38.47&\hl{\textbf{42.19}} & est\_Latn&42.39&\hl{\textbf{46.11}} & eus\_Latn&41.21&\hl{\textbf{44.4}} \\
ewe\_Latn & \hl{\textbf{19.29}}&19.06 & fin\_Latn&42.96&\hl{\textbf{47.34}} & fon\_Latn & \hl{\textbf{13.41}}&13.24 & glg\_Latn&42.29&\hl{\textbf{46.29}} & grn\_Latn&24.82&\hl{\textbf{25.43}} \\
guj\_Gujr&31.84&\hl{\textbf{34.51}} & hne\_Deva&29.56&\hl{\textbf{30.68}} & hrv\_Latn&41.98&\hl{\textbf{46.32}} & ilo\_Latn&39.6&\hl{\textbf{42.04}} & ind\_Latn&49.16&\hl{\textbf{53.94}} \\
ita\_Latn&43.36&\hl{\textbf{47.28}} & kab\_Latn & \hl{\textbf{18.8}}&18.53 & kac\_Latn & \hl{\textbf{20.99}}&20.59 & kan\_Knda&33.04&\hl{\textbf{35.91}} & kas\_Arab&18.55&\hl{\textbf{18.61}} \\
kas\_Deva&15.25&\hl{\textbf{15.67}} & kat\_Geor&38.23&\hl{\textbf{40.15}} & kea\_Latn&31.77&\hl{\textbf{32.95}} & kmr\_Latn&28.39&\hl{\textbf{28.97}} & kon\_Latn & \hl{\textbf{26.63}}&26.4 \\
kor\_Hang&23.67&\hl{\textbf{27.85}} & lao\_Laoo&21.92&\hl{\textbf{22.96}} & lin\_Latn&34.6&\hl{\textbf{35.08}} & lit\_Latn&41.27&\hl{\textbf{45.91}} & lmo\_Latn&28.13&\hl{\textbf{29.46}} \\
ltz\_Latn&37.25&\hl{\textbf{40.18}} & lug\_Latn&28.77&\hl{\textbf{29.4}} & luo\_Latn & \hl{\textbf{24.04}}&23.58 & lus\_Latn&28.41&\hl{\textbf{29.05}} & mag\_Deva&31.85&\hl{\textbf{33.36}} \\
mai\_Deva&32.64&\hl{\textbf{34.48}} & mal\_Mlym&33.35&\hl{\textbf{35.19}} & mar\_Deva&35.55&\hl{\textbf{37.93}} & min\_Latn&33.55&\hl{\textbf{35.11}} & mkd\_Cyrl&42.78&\hl{\textbf{46.41}} \\
mlt\_Latn&39.32&\hl{\textbf{42.66}} & mya\_Mymr&33.82&\hl{\textbf{34.75}} & nld\_Latn&43.18&\hl{\textbf{46.85}} & nno\_Latn&41.2&\hl{\textbf{45.15}} & nob\_Latn&42.66&\hl{\textbf{46.2}} \\
pbt\_Arab&28.0&\hl{\textbf{29.61}} & pol\_Latn&39.07&\hl{\textbf{42.88}} & prs\_Arab&36.32&\hl{\textbf{39.37}} & quy\_Latn & \hl{\textbf{27.19}}&27.12 & run\_Latn&32.32&\hl{\textbf{33.95}} \\
rus\_Cyrl&41.19&\hl{\textbf{45.64}} & sag\_Latn & \hl{\textbf{15.51}}&15.31 & san\_Deva&25.63&\hl{\textbf{26.42}} & sat\_Olck & \hl{\textbf{15.09}}&14.84 & scn\_Latn&33.87&\hl{\textbf{36.14}} \\
sin\_Sinh&26.95&\hl{\textbf{28.23}} & slk\_Latn&39.56&\hl{\textbf{44.11}} & sna\_Latn&38.08&\hl{\textbf{39.71}} & snd\_Arab&32.23&\hl{\textbf{35.04}} & som\_Latn&37.31&\hl{\textbf{39.34}} \\
spa\_Latn&41.63&\hl{\textbf{45.04}} & srd\_Latn&33.17&\hl{\textbf{33.97}} & srp\_Cyrl&39.95&\hl{\textbf{44.62}} & ssw\_Latn&31.16&\hl{\textbf{32.14}} & sun\_Latn&39.64&\hl{\textbf{42.83}} \\
swe\_Latn&44.77&\hl{\textbf{50.2}} & swh\_Latn&46.01&\hl{\textbf{49.52}} & szl\_Latn&30.11&\hl{\textbf{31.39}} & tat\_Cyrl&37.31&\hl{\textbf{40.06}} & tgk\_Cyrl&35.62&\hl{\textbf{37.94}} \\
tgl\_Latn&48.18&\hl{\textbf{50.73}} & tha\_Thai&41.28&\hl{\textbf{44.1}} & tpi\_Latn&35.17&\hl{\textbf{35.32}} & twi\_Latn&27.28&\hl{\textbf{28.16}} & uig\_Arab&29.46&\hl{\textbf{30.7}} \\
ukr\_Cyrl&40.25&\hl{\textbf{44.27}} & urd\_Arab&35.43&\hl{\textbf{38.62}} & uzn\_Latn&42.67&\hl{\textbf{45.03}} & war\_Latn&43.92&\hl{\textbf{46.35}} & zho\_Hans&22.72&\hl{\textbf{27.62}} \\
        \bottomrule  
    \end{tabular}  
\end{adjustbox}  
\caption{Results on GPT-4o-mini in chrF scores on 100 languages from English into other languages.}  
\label{tab:language_metrics5}  
\end{table*}

\begin{table*}[ht]  
\centering  
\small  
\begin{adjustbox}{max width=\textwidth}  
    \begin{tabular}{L|CC|L|CC|L|CC|L|CC|L|CC}  
        \toprule  
        \textbf{Language} & \textbf{Low} & \textbf{High} &   
        \textbf{Language} & \textbf{Low} & \textbf{High} &   
        \textbf{Language} & \textbf{Low} & \textbf{High} &   
        \textbf{Language} & \textbf{Low} & \textbf{High} &   
        \textbf{Language} & \textbf{Low} & \textbf{High} \\
        \midrule  

acm\_Arab & \hl{\textbf{79.87}}&79.66 & ajp\_Arab&77.15&\hl{\textbf{78.26}} & arb\_Arab&81.29&\hl{\textbf{83.8}} & ars\_Arab&81.23&\hl{\textbf{81.5}} & bel\_Cyrl&79.82&\hl{\textbf{82.32}} \\
bul\_Cyrl&83.77&\hl{\textbf{87.16}} & cym\_Latn&75.39&\hl{\textbf{79.75}} & dan\_Latn&83.2&\hl{\textbf{86.9}} & ell\_Grek&83.91&\hl{\textbf{86.36}} & est\_Latn&85.22&\hl{\textbf{87.48}} \\
eus\_Latn&77.48&\hl{\textbf{80.95}} & fin\_Latn&86.45&\hl{\textbf{89.23}} & glg\_Latn&79.21&\hl{\textbf{82.63}} & hrv\_Latn&85.04&\hl{\textbf{87.91}} & ind\_Latn&85.61&\hl{\textbf{88.27}} \\
ita\_Latn&81.69&\hl{\textbf{84.54}} & kor\_Hang&85.16&\hl{\textbf{87.18}} & lao\_Laoo&48.57&\hl{\textbf{51.22}} & lit\_Latn&84.51&\hl{\textbf{87.16}} & mar\_Deva&65.34&\hl{\textbf{68.35}} \\
mkd\_Cyrl&81.63&\hl{\textbf{84.46}} & nld\_Latn&81.21&\hl{\textbf{84.58}} & pol\_Latn&83.65&\hl{\textbf{86.67}} & rus\_Cyrl&83.43&\hl{\textbf{86.25}} & san\_Deva&63.69&\hl{\textbf{64.82}} \\
slk\_Latn&83.36&\hl{\textbf{86.67}} & snd\_Arab&72.46&\hl{\textbf{75.13}} & som\_Latn&76.32&\hl{\textbf{77.52}} & spa\_Latn&79.84&\hl{\textbf{82.72}} & srp\_Cyrl&81.37&\hl{\textbf{85.24}} \\
swe\_Latn&83.02&\hl{\textbf{86.87}} & swh\_Latn&78.6&\hl{\textbf{80.88}} & tha\_Thai&83.69&\hl{\textbf{85.28}} & uig\_Arab&63.02&\hl{\textbf{65.16}} & ukr\_Cyrl&84.11&\hl{\textbf{86.82}} \\
urd\_Arab&77.2&\hl{\textbf{79.5}} & uzn\_Latn&84.76&\hl{\textbf{86.58}} \\

        \bottomrule  
    \end{tabular}  
\end{adjustbox}  
\caption{Results on GPT-4o-mini in COMET scores on 37 supported languages from English into other languages.}  
\label{tab:language_metrics6}  
\end{table*}

\begin{table*}[thb!]
\centering
    \setlength\tabcolsep{10pt}
    \setlength\extrarowheight{0pt}
\begin{tabular}{l|ccc}
\hline
\textbf{Models} & \textbf{GPT-4o-mini} & \textbf{DeepSeek-V3} & \textbf{Qwen2.5-14B-Instruct}\\
\hline
 \multicolumn{4}{c}{\textit{Tool Selection Accuracy}}\\
\hline
Low-frequency partition & 0.6053 & 0.6140 &  0.6316\\
High-frequency partition & \textbf{0.6667} & \textbf{0.6404} & \textbf{0.6667}\\
\hline
 \multicolumn{4}{c}{\textit{Accuracy with Correct Tool Using}}\\
\hline
Low-frequency partition & 0.4386 & 0.4649 &  0.4298\\
High-frequency partition & \textbf{0.4912} & \textbf{0.4737} & \textbf{0.4474}\\
\hline
\end{tabular}
\caption{\label{tcresults}
Results reported in accuracy on the partition of TC. We see that the high-frequency partition gives better results on all baseline models.
}
\end{table*}

\begin{figure}[ht!]
\begin{center}
\vspace{0mm}
\centerline{
\includegraphics[width=7.5cm]{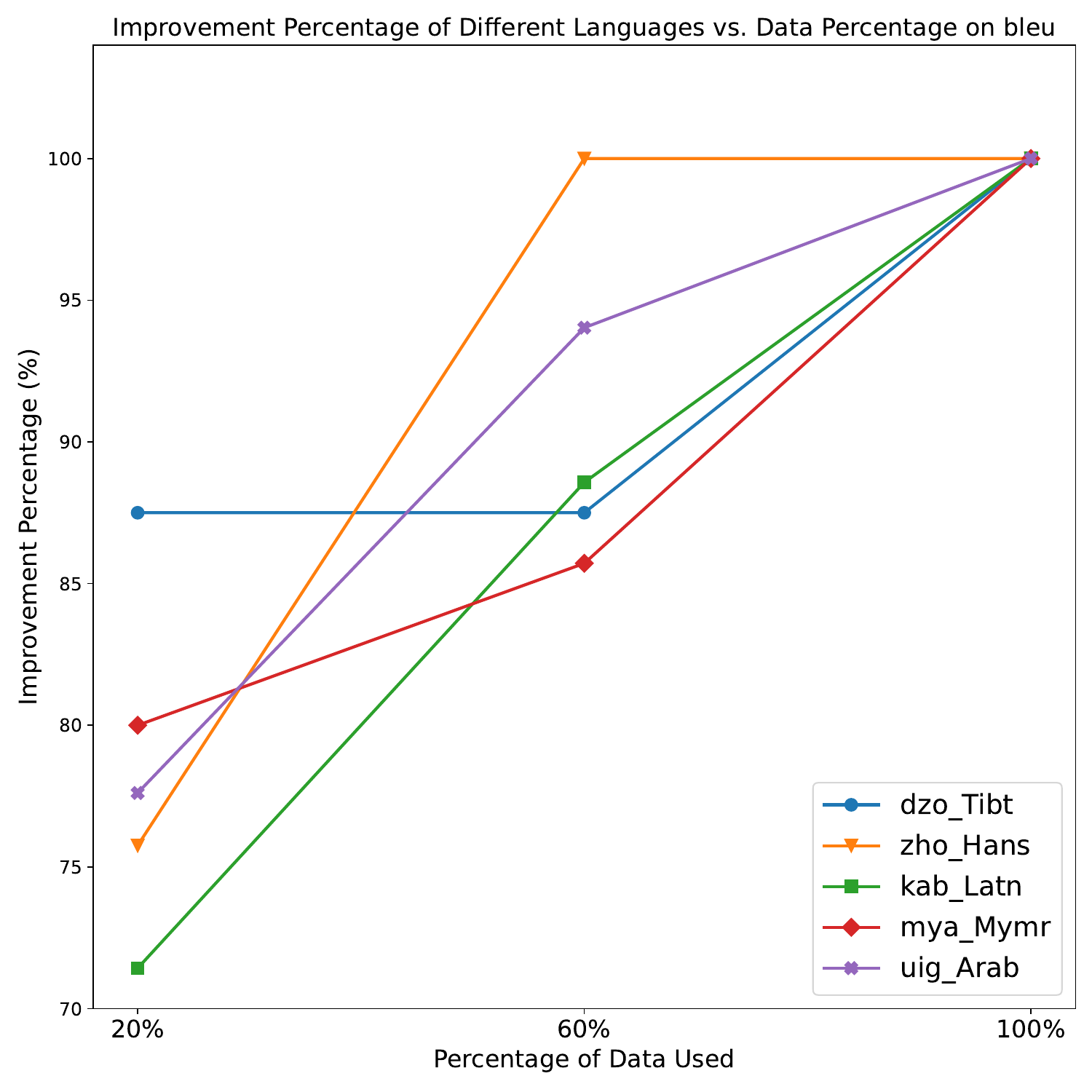}}
    \caption{The figure that demonstrates the relationship between performance percentage and the amount of data used for TFD. We can see that with more data used, the performance improvement increases.}
    \label{fig:tfd2}
\end{center}
\vspace{-5mm}
\end{figure}

\begin{table*}[ht!]
\centering
    \setlength\tabcolsep{10pt}
    \setlength\extrarowheight{1pt}
\begin{tabular}{l|cccc}
\hline
 
\textbf{Models} & \textbf{kea\_Latn} & \textbf{kik\_Latn} & \textbf{lvs\_Latn}&\textbf{pag\_Latn}\\
\hline
\multicolumn{5}{c}{\textit{BLEU}}\\
\hline
 low-frequency & 0.9504	& 0.6983 & 0.7781 & 0.9814\\
 high-frequency & \textbf{1.1528} & \textbf{0.7257} & \textbf{1.2053} & \textbf{1.0204}\\
\hline
\multicolumn{5}{c}{\textit{chrF}}\\
\hline
 low-frequency & 28.6936 & 22.1032 & 29.0109 & 29.8830\\
 high-frequency & \textbf{29.8472} & \textbf{22.9479} & \textbf{29.0681} & \textbf{30.4843}\\
\hline
\end{tabular}
\caption{\label{finetuningplus}
Results of using low-frequency and high-frequency partitions with fine-tuning models with CTFT on translation from English into other languages. COMET is not reported due to unsupported languages. Results indicate that prompting with higher-frequency paraphrases on the model tuned with CTFT is still useful.
}
\end{table*}

\begin{table}[htb]
\centering
\setlength\tabcolsep{8pt}
\setlength\aboverulesep{0pt}\setlength\belowrulesep{0pt}
\setcellgapes{3pt}\makegapedcells
\begin{tabular}{c|c}
\hline
\textbf{Language} & \textbf{Correlation}\\
\hline
ilo\_Latn & 0.9278\\
srp\_Cyrl & 0.8950\\
bho\_Deva&0.9506\\
lao\_Laoo&1.0000\\
mya\_Mymr&1.0000\\
kab\_Latn&1.0000\\
kas\_Deva&1.0000\\
\hline
\end{tabular}
\caption{\label{pptable}
The correlation between textual frequency and the final translation BLEU scores on translating from English into other languages. We compute Pearson correlation coefficients \citep{freedman2007statistics} using numpy.corrcoef().
}
\end{table}

\begin{figure*}[ht!]
\begin{center}
\vspace{0mm}
\centerline{
\includegraphics[width=15cm]{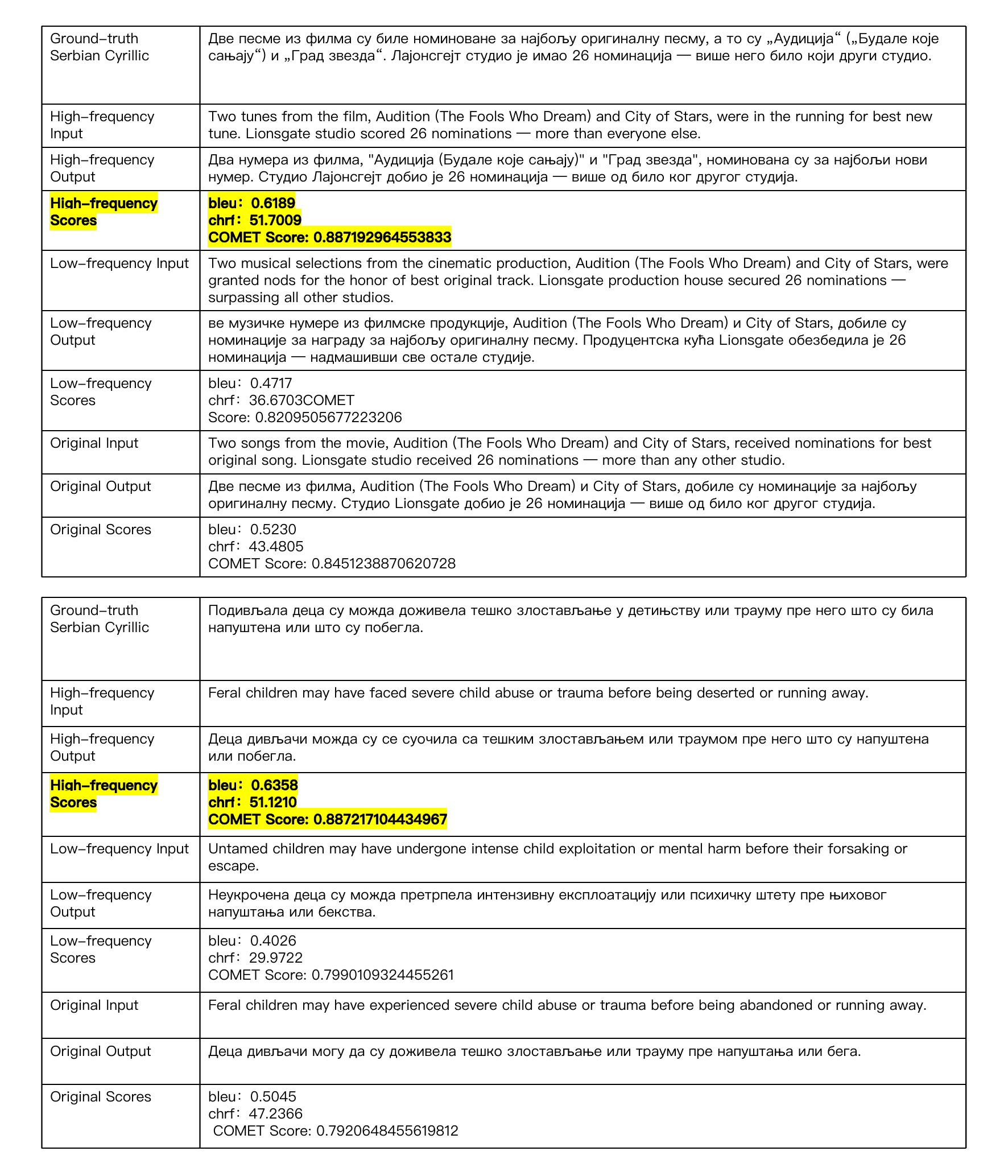}}
    \caption{Case studies on translating following our proposed framework. Best results are bolded and highlighted.}
    \label{fig:casetable}
\end{center}
\vspace{-5mm}
\end{figure*}

\begin{table}
\centering
\setlength\tabcolsep{15pt}
\setlength\aboverulesep{0pt}\setlength\belowrulesep{0pt}
\setcellgapes{3pt}\makegapedcells
\begin{tabular}{l|cc}
\hline
\textbf{Tasks}& \textbf{high-freq} & \textbf{low-freq} \\
\hline
\multicolumn{3}{c}{\textit{Math Reasoning}}\\
\hline
\#. Total & 526 & 526\\
0.0-1.5 & 3 & 60 \\
1.5-2.5 & 225 & 418 \\
2.5-3.5 & 239 & 45 \\
3.5-4.5 & 50 & 3 \\
4.5-5.5 & 9 & 0 \\
\hline
\multicolumn{3}{c}{\textit{Machine Translation}}\\
\hline
\#. Total & 738 & 738\\
1.0-1.5 & 198 & 73\\
1.5-2.0 & 397 & 402\\
2.0-2.5 & 132 & 216\\
2.5-3.0 & 10 & 41\\
3.0-3.5 & 1 & 6\\
\hline
\end{tabular}
\caption{\label{freqstat}
The statistics are based on the TFD calculations: We first statistically calculate the occurrence frequencies of unigrams and bigrams from both the web resources and the generated corpus, then assign different weights to the two corpora, and finally calculate the weighted geometric average of the unigram and bigram frequencies.
}
\end{table}
\begin{table}[htb]
\centering
\setlength\tabcolsep{8pt}
\setlength\aboverulesep{0pt}\setlength\belowrulesep{0pt}
\setcellgapes{3pt}\makegapedcells
\begin{tabular}{c|c|c}
\hline
\textbf{Model Size} & \textbf{Low} & \textbf{High}\\
\hline
0.5b & 0.273& 0.325\\
1.5b&0.442 & 0.484\\
3b & 0.528 & 0.581\\
7b & 0.595 & 0.671\\
14b & 0.600 & 0.690\\
32b & 0.612 & 0.680\\
72b & 0.610 & 0.686\\
\hline
\end{tabular}
\caption{\label{scaling}
The evaluation on different model sizes using qwen-2.5. The results are reported on the task of MR.
}
\end{table}

\begin{table}[htb]
\centering
\setlength\tabcolsep{8pt}
\setlength\aboverulesep{0pt}\setlength\belowrulesep{0pt}
\setcellgapes{3pt}\makegapedcells
\begin{tabular}{c|c}
\hline
\textbf{Hyperparameter} & \textbf{Value}\\
\hline
quantization\_bit & 4 \\  
stage & sft\\  
do\_train & true\\  
finetuning\_type & lora\\ 
lora\_target & all\\
template&qwen\\
cutoff\_len&1024\\
max\_samples&3000\\  
overwrite\_cache & true\\ 
preprocessing\_num\_workers&16\\  
logging\_steps&10\\  
save\_steps&500\\
per\_device\_train\_batch\_size&1\\  
gradient\_accumulation\_steps & 8\\  
learning\_rate & 1.0e-4\\
num\_train\_epochs & 10.0\\  
lr\_scheduler\_type & cosine\\
warmup\_ratio & 0.1\\
bf16 & true\\  
\hline
\end{tabular}
\caption{\label{np}
A list of hyperparameters used in our fine-tuning experiments.
}
\end{table}

\begin{table}[htb]
\centering
\setlength\tabcolsep{8pt}
\setlength\aboverulesep{0pt}\setlength\belowrulesep{0pt}
\setcellgapes{3pt}\makegapedcells
\begin{tabular}{c|c}
\hline
\textbf{Language Class} & \textbf{Number}\\
\hline
0&16\\
1&46\\
2&5\\
3&17\\
4&12\\
5&4\\
\hline
\end{tabular}
\caption{\label{lc}
A list of language classes of the 100 languages used in our experiments. More than half of the languages used in our study are relatively low-resource according to \citet{joshi-etal-2020-state}.
}
\end{table}

\begin{table}[htb]
\centering
\setlength\tabcolsep{8pt}
\setlength\aboverulesep{0pt}\setlength\belowrulesep{0pt}
\setcellgapes{3pt}\makegapedcells
\begin{tabular}{c|c|c}
\hline
\textbf{Metrics} & \textbf{Low} & \textbf{High}\\
\hline
chrF & 18.823 & 32.873\\
ROUGE & 0.175 & 0.310\\
BERTScore & 0.492 & 0.838\\
\hline
\end{tabular}
\caption{\label{chain-of-thought-evaluation}
The evaluation of the chain-of-thought process on the MR partition of our proposed TFPD dataset.
}
\end{table}

\onecolumn
\section{Scope and Proof Strategy}
\label{sec:scope}

This document provides a self-contained formal proof for the Textual
Frequency Law (TFL).  The central claim is:

\medskip
\noindent\textit{When two text sequences express the same meaning (i.e., are
paraphrases), the one with higher sentence-level frequency tends to incur a
lower negative log-likelihood (NLL) loss under a language model trained via
cross-entropy minimisation.}
\medskip

The proof proceeds in two parts.  Part~I (Section~\ref{sec:token}) establishes
the relationship between token-level NLL loss and token frequency rank under
Zipf's law.  Part~II (Section~\ref{sec:sentence}) lifts the token-level result
to the sentence level by introducing a sentence-frequency measure and
accounting for the gap between marginal and conditional token predictions.
Section~\ref{sec:discussion} discusses the relationship between the
mathematical conclusion (loss ordering) and the empirical observation (task
performance ordering).  Section~\ref{sec:limitations} catalogues the
limitations of the theoretical framework.

Throughout, all logarithms are natural logarithms (base $e$; units: nats).

\section{Notation}
\label{sec:notation}

\begin{itemize}[leftmargin=2em, itemsep=2pt]
  \item $V$: vocabulary (finite set of tokens).
  \item $w$: a token in $V$; $w_r$ denotes the token with frequency rank $r$
        ($r=1$ is the most frequent).
  \item $P(w)$: true marginal probability of token $w$ in the training
        distribution.
  \item $Q_\theta(w)$: marginal probability assigned to token $w$ by a language
        model with parameters $\theta$.
  \item $Q_\theta(w \mid c)$: conditional probability of $w$ given context $c$
        under the autoregressive model.
  \item $\ell^{\mathrm{m}}_\theta(w) \triangleq -\ln Q_\theta(w)$: \textbf{marginal}
        token-level NLL loss.
  \item $\ell^{\mathrm{c}}_\theta(x_k \mid x_{<k}) \triangleq
        -\ln Q_\theta(x_k \mid x_1, \dots, x_{k-1})$: \textbf{conditional}
        token-level NLL loss in an autoregressive model.
  \item $x = (x_1, x_2, \dots, x_K)$: a sentence (token sequence) of length $K$.
  \item $\ell_\theta(x) \triangleq \frac{1}{K}\sum_{k=1}^{K}
        \ell^{\mathrm{c}}_\theta(x_k \mid x_{<k})$: average conditional NLL
        loss of sentence $x$ --- the quantity the autoregressive model actually
        computes.
  \item $\sfreq(x)$: sentence-level frequency, defined in
        Assumption~\ref{ass:sfreq}.
  \item $Z = \sum_{n=1}^{|V|} n^{-s}$: Zipf normalisation constant;
        $C \triangleq \ln Z > 0$.
\end{itemize}

\begin{remark}[Marginal vs.\ conditional loss]
\label{rem:marg-cond}
It is essential to distinguish the marginal loss $\ell^{\mathrm{m}}_\theta(w)$
from the conditional loss $\ell^{\mathrm{c}}_\theta(x_k \mid x_{<k})$.  The
Zipf-based analysis in Part~I operates on marginal quantities.  Part~II bridges
to the conditional quantities that autoregressive models actually use, via an
explicit error term.
\end{remark}

\section{Assumptions}
\label{sec:assumptions}

We state four formal assumptions that the proof depends on, followed by one
contextual remark on the training objective.

\begin{assumption}[Zipf's Law for Token Frequencies]
\label{ass:zipf}
The true marginal probability of token $w_r$ with rank $r$ satisfies
\[
  P(w_r) = \frac{r^{-s}}{Z}, \quad s > 0,
  \quad Z = \sum_{n=1}^{|V|} n^{-s}.
\]
\end{assumption}

Zipf's law is a well-documented empirical regularity for the marginal frequency
of tokens aggregated over a large corpus.  It characterises the bulk of the
vocabulary distribution accurately, though deviations occur in the extreme tail
(very rare tokens).  We treat $s$ as a fixed positive constant.

\begin{assumption}[Rank-Dependent Log-Domain Approximation]
\label{ass:approx}
After training, for every token $w_r \in V$ there exists a rank-dependent
bound $\varepsilon(r) \ge 0$ such that
\begin{equation}\label{eq:approx}
  \bigl|\ln Q_\theta(w_r) - \ln P(w_r)\bigr| \le \varepsilon(r).
\end{equation}
\end{assumption}

\begin{remark}[Strength and character of Assumption~\ref{ass:approx}]
\label{rem:approx-strength}
Equation~\eqref{eq:approx} is equivalent to a multiplicative approximation
guarantee:
\[
  e^{-\varepsilon(r)} \le \frac{Q_\theta(w_r)}{P(w_r)} \le e^{\varepsilon(r)},
  \quad \forall\, r.
\]
This is a \emph{pointwise} condition on every token --- considerably stronger
than merely controlling the expected cross-entropy loss.  Standard
cross-entropy training minimises $\E_{w \sim P}[-\ln Q_\theta(w)]$, which
controls the $P$-weighted average loss but does not, by itself, guarantee
pointwise log-domain accuracy for each individual token.

We expect $\varepsilon(r)$ to be small for high-frequency tokens (small $r$),
because these tokens are observed abundantly during training and the model
receives strong gradient signal for them.  For low-frequency tokens (large $r$),
the model may see very few training examples, and $\varepsilon(r)$ is expected
to grow.  All subsequent results are stated in terms of $\varepsilon(r)$, so
the reader can assess the strength of each conclusion as a function of the
model's approximation quality at each frequency tier.

Assumption~\ref{ass:approx} is \textbf{not derivable} from the training
objective alone.  It is an empirical hypothesis about the outcome of
training --- motivated by the fact that cross-entropy minimisation encourages
$Q_\theta \to P$, but not logically entailed by it.

\smallskip\noindent\textbf{Empirical motivation.}
Although no existing study directly measures the pointwise bound $\varepsilon(r)$
as a function of rank, several independent lines of evidence support the
plausibility of Assumption~\ref{ass:approx}:

\begin{enumerate}[leftmargin=2em, label=(\alph*)]
  \item \textit{LLM token distributions follow Zipf's law.}
        \citet{mikhaylovskiy-2025-zipfs} shows that text generated by
        large language models obeys Zipf's law, though the fit quality depends
        on decoding temperature.  This indicates that the model's output
        distribution $Q_\theta$ preserves the rank--frequency structure of the
        training distribution $P$, a necessary (though not sufficient) condition
        for small $\varepsilon(r)$.

  \item \textit{LLMs encode token frequency in their prediction heads.}
        \citet{kobayashi-etal-2023-transformer} demonstrate that the
        bias terms in the prediction head of Transformer language models (BERT
        and GPT-2) significantly reflect corpus word frequency, effectively
        encoding a frequency prior consistent with logit adjustment in
        long-tail learning.  This suggests that the model's internal
        mechanism is structured in a way that facilitates accurate
        frequency-based predictions.

  \item \textit{Frequency modulates model--human surprisal alignment.}
        \citet{oh-etal-2024-frequency} find that word frequency systematically
        modulates the gap between LLM surprisal estimates and human reading
        times, with larger models predicting low-frequency words ``too
        accurately'' relative to human expectations.  This is consistent with
        the view that well-trained models achieve small $\varepsilon(r)$ for
        high-frequency tokens and progressively larger errors in the tail.

  \item \textit{Downstream performance correlates with Zipfian fit.}
        \citet{he-etal-2025-pre} show that pre-trained models
        consistently achieve optimal downstream performance when the vocabulary
        size is chosen so that the resulting token frequency distribution
        follows Zipf's law.  Their experiments across NLP, genomics, and
        chemistry establish a link between Zipfian alignment at the
        tokenisation level and model quality, reinforcing the broader premise
        that power-law regularity in the token distribution --- a key
        ingredient of Assumption~\ref{ass:approx} --- is conducive to
        effective language modelling.
\end{enumerate}

\noindent These findings collectively support the hypothesis that $\varepsilon(r)$ is
small for high-frequency tokens and grows with rank, but a direct empirical
characterisation of the pointwise bound remains an open problem.
\end{remark}

\begin{assumption}[Bounded Marginal--Conditional Discrepancy]
\label{ass:context}
For each token $x_k$ in a sentence $x = (x_1, \dots, x_K)$, define the
\textbf{contextual discrepancy}:
\[
  \eta_{x_k} \triangleq
  \ell^{\mathrm{c}}_\theta(x_k \mid x_{<k})
  - \ell^{\mathrm{m}}_\theta(x_k)
  = \ln Q_\theta(x_k) - \ln Q_\theta(x_k \mid x_1, \dots, x_{k-1}).
\]
We assume that for each sentence $x$, the average contextual discrepancy is
bounded:
\[
  |\bar{\eta}_x| \le \eta_x, \quad\text{where}\quad
  \bar{\eta}_x \triangleq \frac{1}{K} \sum_{k=1}^{K} \eta_{x_k},
\]
and $\eta_x \ge 0$ is a sentence-dependent bound.
\end{assumption}

\begin{remark}[Nature of $\eta_{x_k}$]
\label{rem:eta}
The sign and magnitude of $\eta_{x_k}$ depend on how informative the context
$x_{<k}$ is for predicting $x_k$:
\begin{itemize}[leftmargin=2em]
  \item $\eta_{x_k} < 0$: the context makes $x_k$ \emph{more} predictable than
        its marginal frequency suggests (conditional probability exceeds
        marginal).  This is typical for high-frequency function words in
        predictable contexts (e.g., ``of'' after ``United States'').
  \item $\eta_{x_k} > 0$: the context makes $x_k$ \emph{less} predictable
        (e.g., a token that is common in isolation but surprising in the given
        context).
  \item $\eta_{x_k} \approx 0$: the context is approximately uninformative for
        $x_k$.
\end{itemize}
For sentences composed of common, high-frequency tokens (the ``high-frequency
paraphrases'' central to TFL), many constituent tokens have highly predictable
collocations, so $\eta_{x_k}$ tends to be negative.  This directional tendency
is \emph{favourable} to TFL: it means the actual conditional loss is
systematically lower than the marginal-based estimate for high-frequency
sentences.  However, we do not rely on this tendency in the proof; instead, we
use the conservative absolute bound $|\bar{\eta}_x| \le \eta_x$.

Note that $\eta_x$ is \textbf{sentence-dependent}: different sentences may have
different bounds.  We do not assume a single universal bound across all
sentences.
\end{remark}

\begin{assumption}[Sentence Frequency via Geometric Mean of Token Frequencies]
\label{ass:sfreq}
The sentence-level frequency of $x = (x_1, \dots, x_K)$ is defined as
\[
  \sfreq(x) \triangleq \Bigl(\prod_{k=1}^{K} P(x_k)\Bigr)^{1/K},
\]
or equivalently in log-space:
\begin{equation}\label{eq:sfreq-log}
  \ln \sfreq(x) = \frac{1}{K}\sum_{k=1}^{K} \ln P(x_k).
\end{equation}
\end{assumption}

This definition treats sentence frequency as the geometric mean of marginal
token frequencies, corresponding to a unigram model for the sentence
probability.  It ignores word order and inter-token dependencies, which is a
deliberate simplification: the goal is a tractable frequency measure that
correlates with how ``common'' the constituent vocabulary of a sentence is.  For
comparing paraphrases with identical meaning but different word choices, this
measure captures precisely the relevant variation --- the frequency tier of the
vocabulary used.

\begin{remark}[Role of the training objective]
\label{rem:ce}
Standard language model training minimises the expected negative log-likelihood:
$\min_\theta \E_{w \sim P}[-\ln Q_\theta(w)]$.  This training objective
\emph{motivates} Assumption~\ref{ass:approx}: under ideal conditions with
sufficient capacity and data, the minimiser satisfies $Q_\theta(w) = P(w)$
for all $w$, which would give $\varepsilon(r) = 0$ everywhere.  In practice,
finite data and model capacity lead to nonzero $\varepsilon(r)$, particularly
for low-frequency tokens.  The training objective does \textbf{not} appear as a
formal assumption because the proof does not directly invoke it; it serves as
the background justification for why Assumption~\ref{ass:approx} is plausible.
\end{remark}

\section{Part I: Token-Level Results}
\label{sec:token}

\subsection{Step 1: Self-Information under Zipf's Law}

By Assumption~\ref{ass:zipf}, the self-information (ideal NLL) of token $w_r$ is
\begin{align}
  -\ln P(w_r)
  &= -\ln\!\left(\frac{r^{-s}}{Z}\right) \notag\\
  &= -(-s \ln r - \ln Z) \notag\\
  &= s \ln r + \ln Z. \label{eq:selfinfo}
\end{align}
Setting $C \triangleq \ln Z > 0$:
\begin{equation}\label{eq:selfinfo-short}
  -\ln P(w_r) = s \ln r + C.
\end{equation}
This shows that the ideal NLL is \emph{affine} in $\ln r$ with slope $s$ and
intercept $C$.

\subsection{Step 2: Model Loss Bounded by Approximation Error}

By Assumption~\ref{ass:approx}:
\[
  -\varepsilon(r) \le \ln Q_\theta(w_r) - \ln P(w_r) \le \varepsilon(r).
\]
Multiplying through by $-1$ (which reverses the inequalities):
\[
  -\ln P(w_r) - \varepsilon(r) \le -\ln Q_\theta(w_r) \le -\ln P(w_r) + \varepsilon(r).
\]
Defining $\ell^{\mathrm{m}}_\theta(w_r) \triangleq -\ln Q_\theta(w_r)$, we can
write:
\begin{equation}\label{eq:token-decomp}
  \ell^{\mathrm{m}}_\theta(w_r) = -\ln P(w_r) + \delta_{w_r},
  \qquad |\delta_{w_r}| \le \varepsilon(r),
\end{equation}
where $\delta_{w_r} \triangleq -\ln Q_\theta(w_r) - (-\ln P(w_r))
= \ln P(w_r) - \ln Q_\theta(w_r)$ is the signed approximation error for token
$w_r$.

\subsection{Step 3: Semi-Log Linear Relationship}

Substituting~\eqref{eq:selfinfo-short} into~\eqref{eq:token-decomp}:
\begin{equation}\label{eq:semi-log}
  \ell^{\mathrm{m}}_\theta(w_r)
  = s \ln r + C + \delta_{w_r},
  \qquad |\delta_{w_r}| \le \varepsilon(r).
\end{equation}

\begin{theorem}[Token-Level Semi-Log Linearity]
\label{thm:token}
Under Assumptions~\ref{ass:zipf} and~\ref{ass:approx}, the marginal token-level
NLL loss satisfies
\[
  \ell^{\mathrm{m}}_\theta(w_r) = s \ln r + C + \delta_{w_r},
  \quad |\delta_{w_r}| \le \varepsilon(r),
\]
where $s > 0$ is the Zipf exponent and $C = \ln Z > 0$.  In the semi-log plane
($x$-axis: $\ln r$; $y$-axis: $\ell^{\mathrm{m}}_\theta$), the relationship is
linear with slope $s$ and intercept $C$, within a rank-dependent error band of
half-width $\varepsilon(r)$.
\end{theorem}

\begin{proof}
Immediate from the chain of equalities in Steps 1--3.
\end{proof}

\begin{remark}[Semi-log vs.\ log-log]
Equation~\eqref{eq:semi-log} is a \emph{semi-log} linear relationship
($\ell^{\mathrm{m}}_\theta$ is affine in $\ln r$), \textbf{not} a log-log
relationship (which would require $\ln \ell^{\mathrm{m}}_\theta$ to be affine in
$\ln r$, i.e., a power law for the loss itself).
\end{remark}

\subsection{Token-Level Monotonicity}

\begin{theorem}[Sufficient Condition for Strict Token-Level Monotonicity]
\label{thm:token-mono}
Let $w_i, w_j$ be two tokens with $r_i < r_j$ (i.e., $P(w_i) > P(w_j)$).
A sufficient condition for
$\ell^{\mathrm{m}}_\theta(w_i) < \ell^{\mathrm{m}}_\theta(w_j)$
is
\begin{equation}\label{eq:strict-mono}
  \varepsilon(r_i) + \varepsilon(r_j) < s \ln\!\left(\frac{r_j}{r_i}\right).
\end{equation}
In the special case of a uniform bound $\varepsilon(r) \equiv \varepsilon$, this
reduces to
\begin{equation}\label{eq:rank-ratio}
  \frac{r_j}{r_i} > e^{2\varepsilon/s}.
\end{equation}
\end{theorem}

\begin{proof}
We require the worst-case upper bound of $\ell^{\mathrm{m}}_\theta(w_i)$ to be
strictly less than the worst-case lower bound of
$\ell^{\mathrm{m}}_\theta(w_j)$:
\[
  \bigl(s \ln r_i + C + \varepsilon(r_i)\bigr)
  < \bigl(s \ln r_j + C - \varepsilon(r_j)\bigr).
\]
Cancelling $C$ and rearranging:
\[
  \varepsilon(r_i) + \varepsilon(r_j) < s(\ln r_j - \ln r_i)
  = s \ln\!\left(\frac{r_j}{r_i}\right).
\]
When $\varepsilon(r) \equiv \varepsilon$, this becomes $2\varepsilon < s \ln(r_j/r_i)$,
i.e., $r_j/r_i > e^{2\varepsilon/s}$.
\end{proof}

\begin{remark}[When monotonicity fails]
\label{rem:mono-fail}
For adjacent-rank tokens ($r_j = r_i + 1$), the rank ratio is
$1 + 1/r_i \to 1$ as $r_i \to \infty$, so the left-hand side of
\eqref{eq:strict-mono} approaches zero while the right-hand side remains
positive but also approaches zero (as $\ln(1+1/r_i) \approx 1/r_i$).
Condition~\eqref{eq:strict-mono} fails whenever the approximation error
exceeds the Zipf-induced gap.  Strict ordering between tokens of similar
frequency \textbf{cannot be guaranteed} in the tail of the distribution.
This is an inherent limitation: cross-entropy training provides diminishing
approximation quality for rarer tokens.
\end{remark}

\section{Part II: Sentence-Level Extension}
\label{sec:sentence}

This part bridges the token-level results to the sentence level.

\subsection{Setup}

Let $x = (x_1, \dots, x_K)$ and $x' = (x'_1, \dots, x'_{K'})$ be two
sentences.  Their sentence-level losses (as computed by an autoregressive model)
are
\begin{align}
  \ell_\theta(x) &= \frac{1}{K}\sum_{k=1}^{K}
  \ell^{\mathrm{c}}_\theta(x_k \mid x_{<k}), \label{eq:sent-loss-x}\\
  \ell_\theta(x') &= \frac{1}{K'}\sum_{k=1}^{K'}
  \ell^{\mathrm{c}}_\theta(x'_k \mid x'_{<k}). \label{eq:sent-loss-xp}
\end{align}

Their log sentence-frequencies under Assumption~\ref{ass:sfreq} are
\begin{align*}
  \ln\sfreq(x) &= \frac{1}{K}\sum_{k=1}^{K} \ln P(x_k), &
  \ln\sfreq(x') &= \frac{1}{K'}\sum_{k=1}^{K'} \ln P(x'_k).
\end{align*}

Note that $-\ln\sfreq(x) = \frac{1}{K}\sum_{k=1}^{K}(-\ln P(x_k))$, i.e., the
negative log sentence-frequency equals the average ideal marginal NLL.

\subsection{Step 4: Decomposing Sentence-Level Loss}
\label{step:sent-bound}

For each token $x_k$ with rank $r_k$, the conditional loss can be decomposed
as follows:
\begin{align}
  \ell^{\mathrm{c}}_\theta(x_k \mid x_{<k})
  &= \underbrace{-\ln P(x_k)}_{\text{ideal marginal NLL}}
   + \underbrace{\delta_{x_k}}_{\text{marginal approx.\ error}}
   + \underbrace{\eta_{x_k}}_{\text{contextual discrepancy}}, \label{eq:token-full}
\end{align}
where:
\begin{itemize}[leftmargin=2em]
  \item $\delta_{x_k} = \ell^{\mathrm{m}}_\theta(x_k) - (-\ln P(x_k))
        = \ln P(x_k) - \ln Q_\theta(x_k)$, with $|\delta_{x_k}| \le
        \varepsilon(r_k)$ by Assumption~\ref{ass:approx};
  \item $\eta_{x_k} = \ell^{\mathrm{c}}_\theta(x_k \mid x_{<k})
        - \ell^{\mathrm{m}}_\theta(x_k)
        = \ln Q_\theta(x_k) - \ln Q_\theta(x_k \mid x_{<k})$,
        the contextual discrepancy from Assumption~\ref{ass:context}.
\end{itemize}

\textit{Verification.} Adding the three terms on the right-hand side
of~\eqref{eq:token-full}:
\begin{equation}
\begin{split}
  -\ln P(x_k) + [\ln P(x_k) - \ln Q_\theta(x_k)]
  + [\ln Q_\theta(x_k) - \ln Q_\theta(x_k \mid x_{<k})]
  \\= -\ln Q_\theta(x_k \mid x_{<k})
  = \ell^{\mathrm{c}}_\theta(x_k \mid x_{<k}). \quad\checkmark
\end{split}
\end{equation}

Averaging~\eqref{eq:token-full} over all tokens in $x$:
\begin{align}
  \ell_\theta(x)
  &= \frac{1}{K}\sum_{k=1}^{K} \ell^{\mathrm{c}}_\theta(x_k \mid x_{<k})
  \notag\\
  &= \underbrace{\frac{1}{K}\sum_{k=1}^{K}(-\ln P(x_k))}_{=-\ln\sfreq(x)}
   + \underbrace{\frac{1}{K}\sum_{k=1}^{K}\delta_{x_k}}_{\bar{\delta}_x}
   + \underbrace{\frac{1}{K}\sum_{k=1}^{K}\eta_{x_k}}_{\bar{\eta}_x}.
  \label{eq:sent-decomp}
\end{align}

Define the average marginal approximation bound:
\[
  \bar{\varepsilon}_x \triangleq \frac{1}{K}\sum_{k=1}^{K}\varepsilon(r_k).
\]
By the triangle inequality, $|\bar{\delta}_x| \le \bar{\varepsilon}_x$.
By Assumption~\ref{ass:context}, $|\bar{\eta}_x| \le \eta_x$.

Therefore:
\begin{equation}\label{eq:sent-bound}
  \boxed{
    \ell_\theta(x) = -\ln\sfreq(x) + \bar{\delta}_x + \bar{\eta}_x,
    \quad
    |\bar{\delta}_x| \le \bar{\varepsilon}_x,
    \quad
    |\bar{\eta}_x| \le \eta_x.
  }
\end{equation}

\begin{remark}[Tightness of the bound after averaging]
\label{rem:tightness}
The bound $|\bar{\delta}_x| \le \bar{\varepsilon}_x$ is worst-case
(triangle inequality).  If the token-level errors $\delta_{x_k}$ have
approximately zero mean and are weakly correlated across positions, a
central-limit-type argument gives the tighter practical estimate
$|\bar{\delta}_x| \approx O(\bar{\varepsilon}_x / \sqrt{K})$.  Similarly
for $\bar{\eta}_x$.  Thus the sufficient conditions derived below are
conservative; in practice, the effective threshold for the TFL to hold is
likely smaller by a factor on the order of $1/\sqrt{K}$.
\end{remark}

\subsection{Sentence-Level Results}

\begin{theorem}[Sentence-Level Loss--Frequency Relationship]
\label{thm:sent}
Under Assumptions~\ref{ass:zipf}, \ref{ass:approx}, \ref{ass:context},
and~\ref{ass:sfreq}, the sentence-level NLL loss satisfies
\[
  \ell_\theta(x) = -\ln\sfreq(x) + \bar{\delta}_x + \bar{\eta}_x,
\]
with $|\bar{\delta}_x + \bar{\eta}_x| \le \bar{\varepsilon}_x + \eta_x$.
That is, the sentence-level loss is approximately equal to the negative log
sentence-frequency, up to a total error bounded by
$\bar{\varepsilon}_x + \eta_x$.
\end{theorem}

\begin{proof}
Equation~\eqref{eq:sent-decomp} gives the exact decomposition.
By the triangle inequality:
\[
  |\bar{\delta}_x + \bar{\eta}_x|
  \le |\bar{\delta}_x| + |\bar{\eta}_x|
  \le \bar{\varepsilon}_x + \eta_x. \qedhere
\]
\end{proof}

\begin{theorem}[Textual Frequency Law --- Sufficient Condition]
\label{thm:tfl}
Let $x$ and $x'$ be two paraphrases with $\sfreq(x) > \sfreq(x')$.
A sufficient condition for $\ell_\theta(x) < \ell_\theta(x')$ is
\begin{equation}\label{eq:tfl-cond}
  \ln\frac{\sfreq(x)}{\sfreq(x')}
  > (\bar{\varepsilon}_x + \eta_x) + (\bar{\varepsilon}_{x'} + \eta_{x'}),
\end{equation}
where $\bar{\varepsilon}_x, \eta_x$ and $\bar{\varepsilon}_{x'}, \eta_{x'}$
are the approximation and contextual error bounds for $x$ and $x'$,
respectively.
\end{theorem}

\begin{proof}
By Theorem~\ref{thm:sent}, the worst-case upper bound on $\ell_\theta(x)$
and worst-case lower bound on $\ell_\theta(x')$ are:
\begin{align*}
  \ell_\theta(x) &\le -\ln\sfreq(x) + (\bar{\varepsilon}_x + \eta_x), \\
  \ell_\theta(x') &\ge -\ln\sfreq(x') - (\bar{\varepsilon}_{x'} + \eta_{x'}).
\end{align*}
It suffices to require the upper bound on $\ell_\theta(x)$ to be strictly less
than the lower bound on $\ell_\theta(x')$:
\[
  -\ln\sfreq(x) + (\bar{\varepsilon}_x + \eta_x)
  < -\ln\sfreq(x') - (\bar{\varepsilon}_{x'} + \eta_{x'}).
\]
Rearranging (add $\ln\sfreq(x)$ and $(\bar{\varepsilon}_{x'} + \eta_{x'})$
to both sides):
\[
  (\bar{\varepsilon}_x + \eta_x) + (\bar{\varepsilon}_{x'} + \eta_{x'})
  < \ln\sfreq(x) - \ln\sfreq(x')
  = \ln\frac{\sfreq(x)}{\sfreq(x')},
\]
which is precisely condition~\eqref{eq:tfl-cond}.
\end{proof}

\begin{remark}[Sufficient, not necessary]
\label{rem:sufficient}
Condition~\eqref{eq:tfl-cond} is a sufficient condition.  The TFL may hold
even when this condition is not met, because:
\begin{enumerate}[label=(\roman*)]
  \item The worst-case bounds are conservative --- actual errors may partially
        cancel rather than compound.
  \item The averaging effect across $K$ tokens (Remark~\ref{rem:tightness})
        typically yields a much tighter effective error, on the order of
        $(\bar{\varepsilon}_x + \eta_x)/\sqrt{K}$.
  \item For high-frequency paraphrases, the contextual discrepancy
        $\bar{\eta}_x$ tends to be negative (Remark~\ref{rem:eta}), which
        further reduces the actual sentence loss below the worst-case bound.
\end{enumerate}
\end{remark}

\begin{remark}[Practical magnitude of the condition]
\label{rem:magnitude}
The condition requires the log frequency ratio of the two paraphrases to
exceed the sum of all error bounds.  In practice, paraphrases constructed
by substituting a few content words (e.g., ``deserted'' $\to$ ``abandoned'')
while sharing most function words (``the'', ``was'', ``in'') differ modestly
in sentence frequency.  Whether~\eqref{eq:tfl-cond} is satisfied depends on:
\begin{itemize}[leftmargin=2em]
  \item How many tokens differ, and how large the frequency gap is for those
        tokens.
  \item The model's approximation quality ($\varepsilon(r)$) at the relevant
        frequency tiers.
  \item The magnitude of the marginal--conditional discrepancy ($\eta$).
\end{itemize}
The theorem provides the analytical framework; the empirical validation in the
main paper demonstrates that the TFL holds in practice across a wide range of
settings, suggesting that the error terms are typically small enough for the
condition to be effectively met.
\end{remark}

\section{Discussion: From Loss Ordering to Task Performance}
\label{sec:discussion}

Theorems~\ref{thm:sent} and~\ref{thm:tfl} establish that, under the stated
assumptions, higher-frequency paraphrases incur lower NLL loss.  The empirical
claim of the Textual Frequency Law is stronger: higher-frequency paraphrases
lead to better \emph{task performance} (e.g., higher accuracy in math reasoning,
higher BLEU/chrF in machine translation).  Bridging this gap requires
additional reasoning that we outline here.

\paragraph{For prompting.}
When an LLM is prompted with input $x$, the model generates output
$y = (y_1, \dots, y_T)$ by sampling from or maximising the conditional
distribution $Q_\theta(y \mid x)$.  Lower NLL loss on $x$ means the model
assigns higher probability to the token sequence $x$.  This implies that $x$
falls in a region of the input space where the model's internal representations
are better calibrated --- having been shaped by more training examples with
similar token distributions.  An input that the model ``understands'' better
(assigns higher probability to) is more likely to activate the correct
reasoning pathways and produce accurate outputs.  This argument is plausible
and consistent with the empirical evidence, but it is not a formal proof:
the relationship between input perplexity and output quality depends on the
model's internal mechanism, which is not captured by our framework.

\paragraph{For fine-tuning.}
In fine-tuning, the model optimises $\sum_n \log Q_\theta(y_n \mid x_n)$ over
training pairs $(x_n, y_n)$.  If the model already assigns higher probability
to the input tokens of high-frequency paraphrases, the gradient signal from
these examples is more stable and the effective learning rate for the output
mapping is higher.  Additionally, high-frequency inputs are closer to the
pre-training distribution, reducing the risk of catastrophic forgetting.

\paragraph{Status of this argument.}
The connection from loss ordering to task performance is an \textbf{empirically
motivated hypothesis}, not a theorem.  The formal contribution of this proof is
the loss ordering result (Theorem~\ref{thm:tfl}).  The task performance
connection is supported by extensive experiments in the main paper.

\section{Summary of Results}
\label{sec:summary}

\begin{center}
\renewcommand{\arraystretch}{1.5}
\begin{tabular}{@{}lll@{}}
\toprule
\textbf{Result} & \textbf{Equation} & \textbf{Assumptions Used} \\
\midrule
Token semi-log linearity (Thm.~\ref{thm:token})
  & \eqref{eq:semi-log}
  & \ref{ass:zipf}, \ref{ass:approx} \\
Token strict monotonicity (Thm.~\ref{thm:token-mono})
  & \eqref{eq:strict-mono}
  & \ref{ass:zipf}, \ref{ass:approx} \\
Sentence loss--frequency (Thm.~\ref{thm:sent})
  & \eqref{eq:sent-bound}
  & \ref{ass:zipf}, \ref{ass:approx}, \ref{ass:context}, \ref{ass:sfreq} \\
TFL sufficient condition (Thm.~\ref{thm:tfl})
  & \eqref{eq:tfl-cond}
  & \ref{ass:zipf}, \ref{ass:approx}, \ref{ass:context}, \ref{ass:sfreq} \\
\bottomrule
\end{tabular}
\end{center}

\section{Limitations}
\label{sec:limitations}

We catalogue the limitations of the theoretical framework for full transparency.

\begin{enumerate}[leftmargin=2em, label=\arabic*.]

\item \textbf{Assumption~\ref{ass:approx} is not derivable from the training
objective.}
The pointwise log-domain approximation guarantee is stronger than what
cross-entropy minimisation alone can ensure.  Cross-entropy training controls
the $P$-weighted expected loss, not the per-token log-domain error.  The
assumption is empirically motivated but remains a hypothesis about the outcome
of training.  For low-frequency tokens, $\varepsilon(r)$ may be large, and the
theorem's guarantees weaken accordingly.  While several studies provide
indirect support for the plausibility of this assumption (see
Remark~\ref{rem:approx-strength}), a direct empirical measurement of the
pointwise bound $\varepsilon(r)$ as a function of rank remains an open problem
in the literature.

\item \textbf{Contextual discrepancy $\eta_x$ is difficult to estimate.}
The magnitude of $\eta_{x_k}$ depends on the specific sentence context and the
model's learned conditional distributions.  No general data-independent bound is
available.  In the proof, $\eta_x$ is treated as an axiomatically bounded
quantity.  Empirically, one could estimate $\eta_x$ by comparing marginal and
conditional perplexities on a held-out corpus, but such estimates would be
model- and data-specific.

\item \textbf{The sentence frequency measure is a unigram approximation.}
The geometric-mean definition (Assumption~\ref{ass:sfreq}) ignores word order
and inter-token dependencies.  For paraphrase pairs that differ mainly in word
choice (not syntactic structure), this is a reasonable proxy.  For paraphrases
with substantially different syntactic structures or lengths, the measure may
not fully capture the relevant notion of ``commonness.''

\item \textbf{Sentence length differences.}
When two paraphrases have different lengths $K \neq K'$, the averaging effect
differs: a longer sentence averages over more tokens, which may tighten or
loosen the effective error bounds.  This interaction is not explicitly modelled;
the theorem treats $\bar{\varepsilon}_x$ and $\eta_x$ as given quantities.

\item \textbf{Loss ordering does not formally imply task performance ordering.}
The proven result is $\ell_\theta(x) < \ell_\theta(x')$ (lower NLL loss for
higher-frequency paraphrases).  The claim that this translates to better
downstream task performance (higher accuracy, higher BLEU) is empirically
supported but not formally established within this framework.
See Section~\ref{sec:discussion} for further discussion.

\item \textbf{Semantic equivalence is assumed, not verified.}
The TFL compares paraphrases with ``the same meaning.''  The proof assumes
perfect semantic equivalence; in practice, paraphrasing inevitably introduces
subtle meaning shifts.  A formal treatment would require a semantic similarity
metric, which is beyond the scope of a frequency-based theorem.

\item \textbf{Zipf's law is approximate in the tail.}
The power-law model fits well for the bulk of the vocabulary but may deviate
for extremely rare tokens.  Such deviations are absorbed into $\varepsilon(r)$
in the analysis, but this means the error bound for tail tokens reflects both
the model's approximation error and the inadequacy of the Zipf model itself.

\end{enumerate}

\section{Conclusion}
\label{sec:conclusion}

This document has established, under clearly stated assumptions, that:
\begin{enumerate}[label=(\roman*)]
  \item Token-level NLL loss is semi-log linear in frequency rank
        (Theorem~\ref{thm:token}).
  \item Sentence-level NLL loss is approximately equal to the negative log
        sentence-frequency, with a bounded error term
        (Theorem~\ref{thm:sent}).
  \item When the sentence-frequency ratio between two paraphrases is
        sufficiently large relative to the error bounds, the higher-frequency
        paraphrase provably has lower model loss (Theorem~\ref{thm:tfl}).
\end{enumerate}
These results provide the theoretical foundation for the Textual Frequency Law.
The sufficient condition is conservative; empirical evidence in the main paper
demonstrates that the TFL holds broadly in practice, consistent with the
error terms being small enough for the condition to be effectively satisfied in
typical settings.

\end{document}